\documentclass[preprint,12pt,authoryear]{elsarticle}
\usepackage[margin=1in]{geometry}
\usepackage{hyperref}
\usepackage{amsmath, amssymb, amsthm, enumitem}
\usepackage{amsmath}
\usepackage{algorithm}
\usepackage[table,xcdraw]{xcolor}
\usepackage[normalem]{ulem}
\useunder{\uline}{\ul}{}
\usepackage{graphicx}
\usepackage{subfig}
\usepackage{booktabs}
\usepackage{multirow}
\usepackage{algpseudocode}

\usepackage{amssymb}
\usepackage{amsmath}

\journal{elsarticle}

\begin{document}

\begin{frontmatter}



\title{Learning from Yesterday's Error: An Efficient Online Learning Method for Traffic Demand Prediction} 
\begin{highlights}
\item Propose an online learning method using yesterday's errors to adjust future forecasts.
\item Incorporate exponential smoothing  and mixture of experts for robust adjustments.
\item Introduce adaptive spatiotemporal smoothing to capture correlations across regions and time slots.
\item Achieves high accuracy with low computational cost in experiments across seven datasets.
\end{highlights}

\author[a]{Xiannan Huang}
\author[b]{Quan Yuan}
\author[a,b]{Chao Yang \corref{cor1}}
\cortext[cor1]{Corresponding author (e-mail: tongjiyc@tongji.edu.cn)}
\affiliation[a]{organization={Key Laboratory of Road and Traffic Engineering, Ministry of Education at Tongji University},
            addressline={4800 Cao’an Road}, 
            city={Shanghai},
            postcode={201804}, 
            state={Shanghai},
            country={China}}
\affiliation[b]            {organization={Urban Mobility Institute, Tongji University},
            addressline={1239 Siping Road}, 
            city={Shanghai},
            postcode={200092}, 
            state={Shanghai},
            country={China}}

\begin{abstract}
Accurately predicting short-term traffic demand is critical for intelligent transportation systems. While deep learning models achieve strong performance under stationary conditions, their accuracy often degrades significantly when faced with distribution shifts caused by external events or evolving urban dynamics. Frequent model retraining to adapt to such changes incurs prohibitive computational costs, especially for large-scale or foundation models.
To address this challenge, we propose \textbf{FORESEE} (\textbf{F}orecasting \textbf{O}nline with \textbf{Re}sidual \textbf{S}moothing and \textbf{E}nsemble \textbf{E}xperts), a lightweight online adaptation framework that is \textbf{accurate}, \textbf{robust}, and \textbf{computationally efficient}. FORESEE operates without any parameter updates to the base model. Instead, it corrects today’s forecast in each region using yesterday’s prediction error, stabilized through exponential smoothing guided by a mixture-of-experts mechanism that adapts to recent error dynamics. Moreover, an adaptive spatiotemporal smoothing component is proposed to propagates error signals across neighboring regions and time slots, capturing coherent shifts in demand patterns.
Extensive experiments on seven real-world datasets with three backbone models demonstrate that FORESEE consistently improves prediction accuracy, maintains robustness even when distribution shifts are minimal (avoiding performance degradation), and achieves the lowest computational overhead among existing online methods. By enabling real-time adaptation of traffic forecasting models with negligible computational cost, FORESEE paves the way for deploying reliable, up-to-date prediction systems in dynamic urban environments. We release our code and datasets at \href{https://github.com/xiannanhuang/FORESEE/}{GitHub}. 

\end{abstract}


\begin{keyword}
Traffic demand prediction\sep  Online prediction\sep Dynamically self-adaptive
\end{keyword}
\end{frontmatter}
\section{Introduction}
Accurately predicting the short-term (e.g., 30 minutes to several hours) demand for bicycles, taxis, and other transportation modes across various urban regions is crucial for transportation system. This importance stems from its role in efficiently allocating transport resources, developing intelligent transportation systems, promoting sustainable transport and reducing congestion \citep{liang2025foundation,huang_leveraging_2024,QIU2024104552,bai2019stg2seq}. Consequently, substantial research efforts in recent years have been dedicated to building precise short-term transport demand prediction models \citep{zou2024mt,wei2024multi,yan2024prostformer}.

The prevalent methodology for building a traffic demand prediction model involves utilizing historical data as the training set, proposing a specific model architecture and then training the model on this dataset. After that, the model will be used to predict traffic demand in the future. This approach implicitly assumes that the underlying traffic patterns remain consistent over time, meaning that the relationship learned from the training set is applicable to the test period. However, real-world traffic patterns are constantly changing. Numerous factors contribute to these shifts, such as external events (e.g., pandemics, economic shifts, weather changes) and alterations in land usage (e.g., opening a new shopping center or transit hub). Consequently, the patterns learned from the training data may not be fully suitable for future data \citep{wang2025evaluating}. A simple solution to address these changing patterns is retraining the model frequently, perhaps even daily or hourly. However, this approach could lead to substantial computational burden, especially for large foundation models recently proposed for traffic forecasting \citep{li2024opencity}.  

This challenge raises a key question: how can models be adapted to new traffic patterns with limited computational resources? This problem is commonly known as online prediction problem.

Although online prediction research for traffic forecasting problem is relatively limited \citep{guo2025online}, traffic forecasting can be viewed as a time series prediction task. In the broader field of time series prediction, there exists some research specifically focused on online prediction. The most fundamental online prediction approach is Online Gradient Descent (OGD). Specifically, its workflow involves updating model parameters when a prediction is made and the corresponding true value is observed. The prediction error is used to compute gradients of the model parameters via backpropagation, which then guide the parameter updates. Subsequent research about online prediction for time series tasks primarily focuses on improving this basic OGD method. Key directions for improvement include identifying the most critical parameters to update \citep{guo2025online} (since updating fewer parameters can enhance algorithm stability), and employing more sophisticated update methods, such as experience replay or exponential averaging \citep{lau2025fast}.

\begin{figure}[!h]
    \centering
    \includegraphics[width=0.9\linewidth]{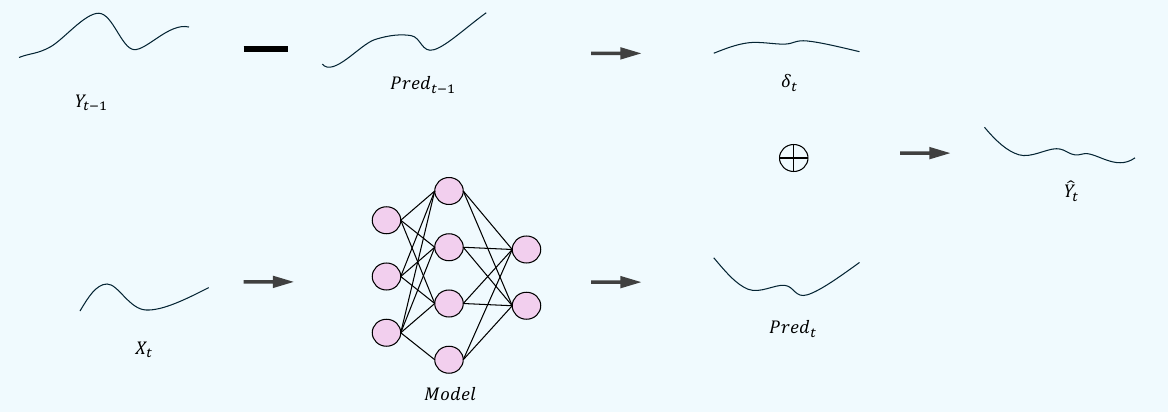}
    \caption{The motivation of our method}
    \label{fig:mov}
\end{figure}

However, traffic forecasting exhibits distinct characteristics such as strong periodicity and spatial correlations between regions. Consequently, directly applying generic time-series online prediction methods—designed for arbitrary non-stationary signals—may be suboptimal for capturing the structured nature of urban mobility shifts.

The proposed method is motivated by an empirical observation about how distribution shifts typically manifest in urban traffic demand. In real-world settings, such shifts are often triggered by persistent structural changes—for instance, the opening of a new commercial complex, long-term urban redevelopment, or large-scale socioeconomic events like economic downturns or public health crises \citep{wang2025evaluating}. These factors tend to induce sustained, directional deviations in demand (e.g., consistently higher or lower than historical patterns) rather than random or rapidly oscillating fluctuations \citep{wang2024impact,zhang2017impact,nian2020impact}.

Under such conditions, the prediction error from the most recent period—particularly the previous day—often reflects the ongoing systematic bias introduced by the shift. This suggests that directly incorporating yesterday’s signed error as a corrective term for today’s forecast can serve as an accurate, robust, and computationally efficient adaptation strategy. As illustrated in Figure \ref{fig:mov}, our approach adjusts today’s prediction by adding a smoothed residual from yesterday. To avoid overreacting to transient noise and maintain robustness under varying degrees of non-stationarity, we apply exponential smoothing enhanced by a mixture-of-experts mechanism that dynamically selects the most appropriate smoothing rate based on recent error dynamics. Moreover, recognizing that such systematic shifts frequently exhibit spatial coherence—neighboring regions often experience similar demand trends due to shared land use \citep{liu2025spatial} or event exposure \citep{hu2021examining}—we introduce an adaptive spatiotemporal smoothing module that propagates error signals across space and time.

We evaluate this strategy on seven real-world traffic demand datasets using three diverse backbone models. Results show that this lightweight adaptation framework consistently improves prediction accuracy under distribution shift, outperforming existing online methods while requiring significantly less computational overhead.

Our contributions are as follows.
\begin{itemize}
    \item We propose FORESEE, a novel online adaptation framework that leverages spatiotemporally smoothed prediction residuals from the previous day to correct future forecasts, avoiding full model retraining.
    \item We introduce an adaptive mixture-of-smoothing-experts mechanism that automatically selects optimal smoothing rates based on recent error dynamics, enhancing robustness to non-stationarity.
    \item We demonstrate through extensive experiments that FORESEE achieves state-of-the-art online adaptation performance across seven real-world datasets and three backbone models, with lower computational overhead than existing parameter-update-based methods.
\end{itemize}
\section{Related Work}
\subsection{Short-term traffic prediction}
Short-term traffic prediction aims to forecast traffic volume in specific regions over future time periods. This task involves both temporal and spatial dimensions. Traffic data form time series in each location, while data from different areas at the same time can be viewed as a graph. Thus, effective models must capture both temporal and spatial patterns.

Researchers have used various methods to handle these aspects. Temporal relationships are often modeled with convolution, attention, or recurrent neural networks \citep{li2023improving,yu2025clear,yu2017spatiotemporal}. Spatial relationships are commonly captured using graph convolution or spatial attention mechanisms \citep{zheng2020gman,guo2019attention,wei2024multi}.

There are three main model structures for integrating temporal and spatial information. One type combines temporal and spatial models into integrated spatiotemporal blocks \citep{yu2017spatiotemporal}. Another adds spatial modules inside temporal models \citep{li2018diffusion}. A third type applies spatial processing to each time step before temporal modeling \citep{zhao2019t}. Some studies also model spatiotemporal features jointly in single modules \citep{song2020spatial}. 

Regrading the training strategy, some approaches include self-supervised learning \citep{ji2023spatio} and meta-learning \citep{pan2019urban} to improve training. Recent work has also explored training with larger datasets across cities. These studies often draw inspiration from large language models, focusing on versatility and zero-shot ability \citep{xu2025gpt4tfp,yuan2024unist,li2024urbangpt}.  

For more details, several review articles offer further insights \citep{jin2024survey,jin2023spatio,du2025traffic}. 

However, most existing studies do not fully explore how to address the online updating problem in traffic prediction task.

\begin{table}[!htbp]
\centering
\small
\caption{The summary of existing online time series prediction methods}
\label{tab:summary}
\begin{tabular}{@{}p{3.5cm}p{5cm}p{7cm}@{}}
\toprule
Method                          & Parameters for   updating                                                                   & Updating style                                                                                                                                                 \\ \midrule
Onenet \citep{wen2023onenet}                   & Ensemble weights of   different models                                                                &  Exponentiated Gradient Descent and offline reinforcement learning                                                                                            \\
FSNet \citep{phamlearning}                    & Parameters in convolutional layers and feature modification weights after convolutional layers & Gradient decent with two-stream exponential moving averaging                                                                                                                          \\
D3A \citep{zhang2024addressing}                     & Full parameters of the model                                                              & 1. Update only when distribution drift happens. 2. Use accumulated and noised data to update the model\\
DSOF \citep{lau2025fast}                     & Parameters in an student model that maps the output of original model and the input to prediction   & Online gradient   decent combined with replay buffer \\
Proceed \citep{zhao2025proactive} & Low-rank scaling coefficients $\alpha,\beta$   for model parameters & Updating using gradients and estimated concept drifts                                                                                                       \\
SOLID \citep{chen2024calibration} & Parameters in the final Layer &1.       Update   only when distribution drift is detected.                                                                                                  2.   using samples with temporal proximity, periodic phase   similarity and input similarity to finetune                                                    \\
ELF \citep{leelightweight} & Parameters in a linear adapter mapping input to prediction and ensemble weight between base model and adapter               & Directly fit for parameters in the adapter; Exponentiated Gradient Descent for ensemble weights                                                                                                                                                                                                                                                \\
ADCSD \citep{guo2025online}                    & The parameters in an adapter that maps the output of original model to prediction               & Online gradient   decent                                                                                                                                       \\ \bottomrule
\end{tabular}
\end{table}

\subsection{Online prediction}

Our work falls within the broader paradigm of online prediction, where models adapt incrementally as new data arrive. 

The online prediction process can be divided to the following two steps: first, updating the model using newly acquired data, and then deploying this updated model at the next timestep. Upon arrival of new data at that subsequent timestep, the model is updated again and then applied to the following timestep, creating a repeating loop. The simplest approach like this cycle is Online Gradient Descent. Most existing online learning methods can be regraded as enhancements or variations of this fundamental OGD procedure. Consequently, the core challenges addressed by these methods involve two parts: determining which parameters to update and how to update them.

When it comes to determining which parameters to update, various strategies are adopted in existing literature. For instance, the OneNet approach \citep{wen2023onenet} adjusts the ensemble weights for two models, whereas FSNet \citep{phamlearning} updates parameters within convolutional layers along with feature modification weights. In contrast, DSOF \citep{lau2025fast} modifies only the weights of the student model. A common and notable trend is the introduction of a compact adapter module—such as a single linear layer \citep{leelightweight} or a small multi-layer perceptron (MLP) \citep{lau2025fast,kim2025battling,medeiros2025accurate}—for online adaptation, instead of updating the entire model. This strategy significantly decreases the number of parameters requiring optimization and enhances learning stability compared to full-model updates with individual data samples.

The second consideration is the methodology for updating the parameters. The most straightforward technique involves computing the loss on a single new sample, obtaining gradients via backpropagation, and then applying a gradient-based update. However, as noted, gradients estimated from a single sample can be highly unstable. To address this limitation, numerous methods enhance the basic scheme. A fundamental improvement is the use of experience replay, where a mini-batch of stored past samples is included in the loss calculation alongside the new data \citep{lau2025fast,chen2024calibration}. This approach aggregates multiple samples for gradient estimation, thereby reducing variance and stabilizing the update process. Beyond replay buffer, methods like FSNet \citep{phamlearning} propose a dual-stream architecture: one stream performs rapid updates based on recent samples, while the other stream updates more slowly using gradients accumulated over a long history. Finally, moving away from gradient-based optimization, some techniques such as ELF \citep{leelightweight} directly infer optimal parameters. In this case, the adapter is formulated as a linear model, allowing its optimal parameters to be computed analytically from a small batch of data, thereby bypassing the need for iterative gradient updates.

We collected recently proposed methods for online time-series forecasting. These methods are categorized in Table \ref{tab:summary} based on which parameters to update and how to update.
\subsection{Test-Time Adaptation}
Test-time adaptation (TTA) represents a closely related paradigm that seeks to adjust pre-trained models during deployment—without full retraining—to cope with distribution shifts. Formally, TTA refers to techniques that dynamically adapt a model using unlabeled test data at inference time to improve performance on the target domain. Initially developed for classification tasks, TTA can be broadly categorized into optimization-based, data-based, and model-based approaches. Optimization-based methods typically recalibrate normalization layer statistics (e.g., in BatchNorm), design unsupervised loss functions such as entropy minimization \citep{wangtent}, or employ mean-teacher frameworks for stable updates \citep{wang2022continual}. Data-based strategies enhance robustness through test-time data augmentation \citep{zhang2022memo} or memory banks for replay \citep{yuan2023robust}. Model-based techniques adapt the architecture itself—e.g., by inserting input transformation layers, adapter modules \citep{liuvida}, or prompt-tuning layers, especially in vision-language models \citep{liu2023meta}.

However, directly transferring these TTA paradigms to traffic demand forecasting (or time series forecasting) is non-trivial. Forecasting models often lack standard normalization layers (e.g., BatchNorm), and image-centric augmentations do not generalize to temporal sequences. In recent years, several works have adapted TTA specifically for time-series settings, including ADCSD \citep{guo2025online} , PETSA \citep{medeiros2025accurate}, and TAFAS \citep{kim2025battling}. These methods align with the two core dimensions discussed above: which parameters to update and how to update them. In terms of parameter selection, all three employ lightweight adapters—but with distinct designs: ADCSD applies its adapter only at the output head, whereas PETSA and TAFAS insert adapters at both input and output stages; notably, PETSA adopts a LoRA-inspired low-rank structure. Regarding update strategies, ADCSD relies on conventional gradient descent using observed history, while PETSA and TAFAS leverage partially observed future sequences (e.g., autoregressive context) to construct more informative adaptation signals. 

\subsection{Limitations of Existing Approaches}
Despite the progress in online time-series prediction and test-time adaptation, existing methods face several limitations when applied to short-term traffic demand forecasting under distribution shift:

First, most approaches rely on gradient-based parameter updates that require backpropagation through the model (or an adapter) at each adaptation step. This incurs significant computational overhead and can suffer from numerical instability—especially in realistic traffic settings where only a single new observation arrives per time step, providing insufficient signal for reliable gradient estimation.

Second, these methods are built on the implicit assumption that adaptation must be achieved by modifying internal model parameters. However, many real-world distribution shifts in traffic demand—such as those caused by new infrastructure, policy changes, or long-term events—manifest primarily as persistent, additive biases across regions and time, rather than fundamental alterations to the underlying spatiotemporal dynamics. In such scenarios, directly correcting predictions using recent residuals is not only sufficient but also more efficient and robust than re-optimizing model weights.

Finally, spatial structure is largely ignored in current online adaptation schemes. While backbone models may incorporate spatial graphs during offline training, the online update process typically treats each region independently (e.g., ELF, ADCSD). This overlooks the fact that systematic errors often exhibit spatial coherence across neighboring zones, which could be leveraged to improve correction quality and data efficiency.

These limitations highlight a critical need for a \textbf{lightweight, spatiotemporally aware} online correction framework. Therefore, we design a residual-based correction framework that operates solely on prediction errors, leveraging their temporal persistence and spatial smoothness to enable robust, efficient, and scalable online forecasting without any model modification.
\section{Method}
\subsection{Temporal Bias Correction via Exponential Smoothing and Expert Ensembling}
We first address the temporal instability of raw prediction errors by modeling systematic bias as an evolving signal. To illustrate this idea, 
consider a problem of forecasting traffic demand across $n$ regions by hour, a straightforward method is to adjust future predictions based on the model's recent error. 

For example, on day 1, the prediction of the model is $pred_1$  (shape of $pred_1$ is $(24, n, 2)$, 24 means the number of hours in one day and 2 means inflow and outflow) and the actual traffic demand is $y_1$ (same shape). The prediction error is: $err_1 =y_1  - pred_1$. Then in day 2, if the prediction of the model is $pred_2$ (shape $(24, n, 2)$), we refine the prediction to:
\begin{equation}
    \hat{y}_2= pred_2 + err_1
\end{equation}
While intuitive, our implementation revealed this basic adjustment method even performed worse than unadjusted predictions in practice. We attribute this failure to high variance in the error estimation. Traffic demands contain inherent noise, making the single-day error an unreliable correction signal. A correction derived solely from the previous day's error amplifies this noise, yielding unstable adjustments. To mitigate this instability and create a robust correction signal, we introduce exponential smoothing. For day $t$ we refine the prediction as:
\begin{equation}
    \hat{y}_t =pred_t  +\delta_t
\end{equation}
The correction for the next day ($\delta_{t+1}$) is calculated using a recursive formula:
\begin{equation}
    \delta_{t+1}  =\alpha \times \delta_t   + (1 -  \alpha)\times err_t,
\end{equation}
where $err_t$ is the prediction error in day $t$ and $\alpha$ ($0 \leq \alpha \leq 1$) is a smoothing factor. this mechanism effectively incorporates weighted contributions from all prior prediction errors, with the influence of older errors decaying exponentially over time. Consequently, $\delta_{t+1}$  represents a stabilized, low-variance estimation of model bias, filtering out daily noise and enabling reliable forecast adjustments.

As a result, determining the optimal exponential smoothing coefficient $\alpha$ presents a significant challenge. Our experiments on New York City bike and taxi datasets, which tested multiple $\alpha$ values, confirm the existence of an $\alpha$ that minimizes prediction error. Setting $\alpha$ too small or too large will increase prediction error. An excessively small $\alpha$ (undersmoothing) yields unstable corrections due to insufficient noise filtering, while an excessively large $\alpha$ (oversmoothing) introduces bias because historical correction terms fail to adapt to dynamically changing patterns. Critically, optimal $\alpha$ values differ across datasets, making it impractical to determine the ideal $\alpha$ before deployment.

We address this using a mixture of experts approach: First we initialize $k$ predictors with different $\alpha$ values: $\alpha_1,…,\alpha_k$, and their weights $w_{1,1},…,w_{k,1}$ , On day $t$, the final prediction is a weighted combination:
\begin{equation}
    \hat{y}_t=\sum_{i=1}^k{w_{i,t}\hat{y}_{i,t}}
\end{equation}
After seeing the actual value $y_t$ on day $t$, we update the weights for day $t+1$ as:
\begin{equation}
    w_{i,t+1}=\frac{w_{i,t}e^{-\eta(\hat{y}_{i,t}-y_t)^2}}{\sum_{j=1}^k{w_{j,t}e^{-\eta(\hat{y}_{j,t}-y_t)^2}}}
\end{equation}
This means if a predictor has a small error, its weight increases in the next day.

We provide some theorems to describe the performance exponential smoothing and mixture of experts in the Section \ref{sec:theo}. Specifically, Theorem \ref{the:pro} demonstrates that if the model's error remains relatively stable over consecutive days—for instance, if the model's prediction is higher than the true value on the previous day, it is likely to be higher on the next day as well—then using the previous day's prediction error to correct the next day's prediction will yield better results than deploying the model directly without such correction. Theorem \ref{theo:ema} indicates that selecting an appropriate smoothing coefficient $\alpha$ can balance the variance and bias of the predictions, thereby further reducing the error during deployment. Finally, Theorem \ref{theo:emas} shows that the long-term performance of our multi-$\alpha$ integration algorithm is nearly equivalent to the optimal $\alpha$.
\subsection{Leveraging Spatiotemporal Coherence for Robust Adjustment}

\begin{figure}[!h]
    \centering
    \includegraphics[width=\linewidth]{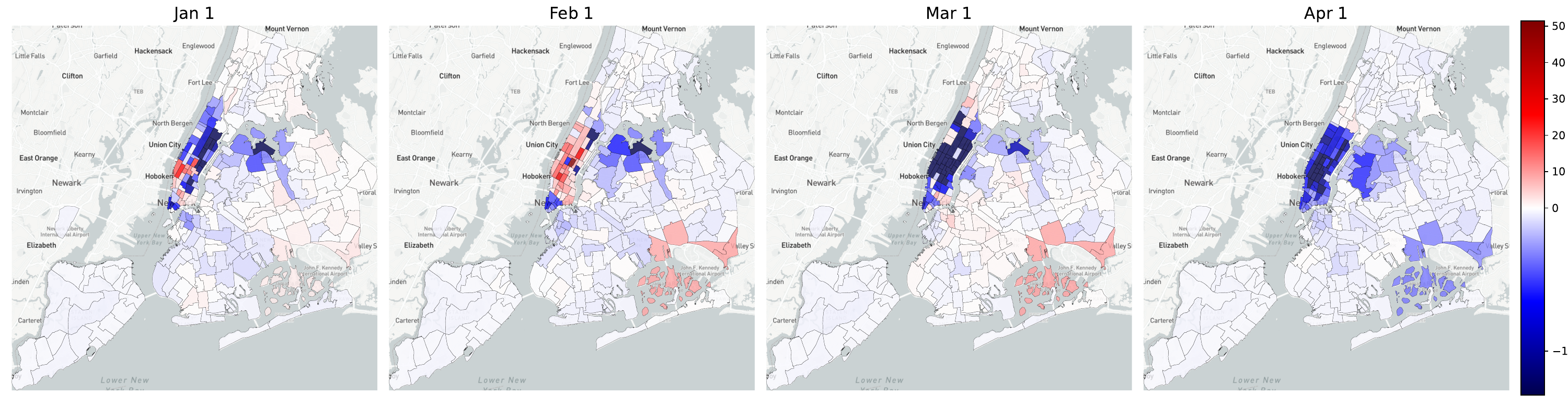}
    \caption{The prediction errors of taxi usage in New York City}
    \label{fig:err}
\end{figure}
\begin{figure}[!h]
    \centering
    \includegraphics[width=\linewidth]{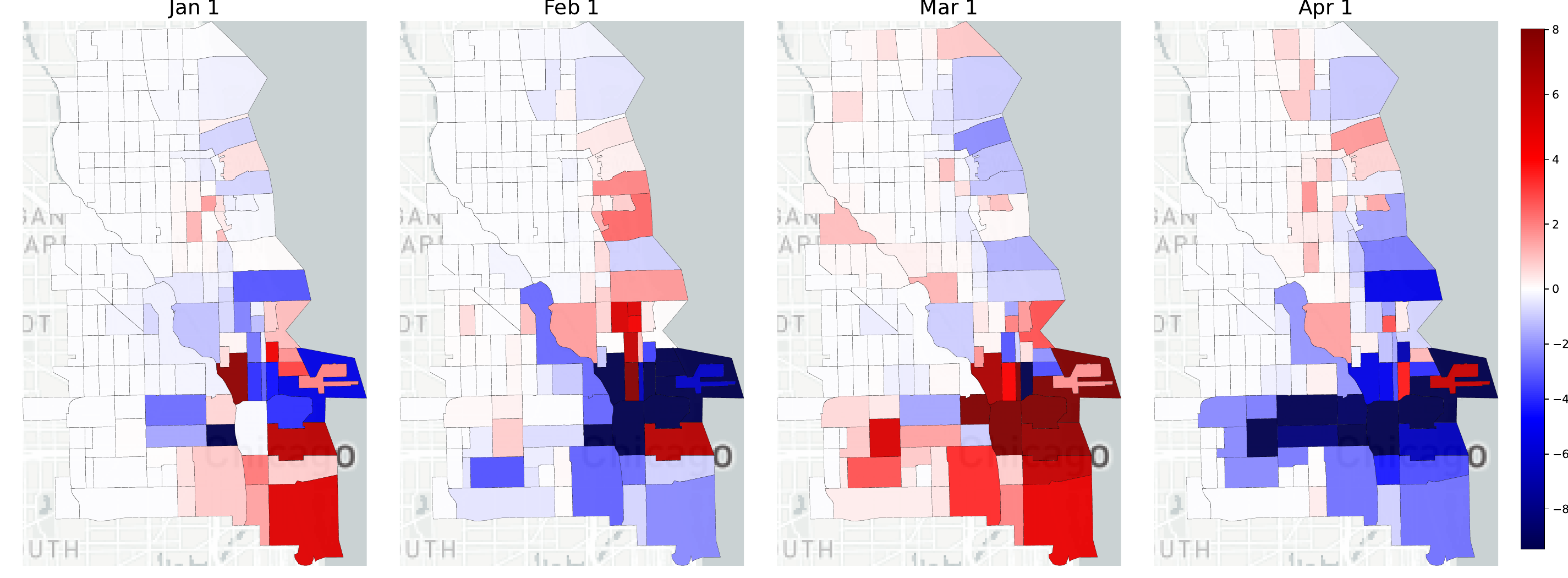}
    \caption{The prediction errors of taxi usage in Chicago City}
    \label{fig:err2}
\end{figure}
The above method can actually be applied to any periodic online time series forecasting problem. However, for traffic prediction, which is a spatiotemporal problem, there might be unique features that can be utilized to achieve better results. So, what is the most important difference between traffic prediction and general time series prediction problem? Traffic prediction involves not only a temporal dimension but also a spatial dimension. This means adjacent areas may share certain characteristics. Therefore, we visualize the prediction error of taxi pickup volumes across all zones in New York City and Chicago at 8:00 AM on January 1, February 1, March 1, and April 1, 2020, in Figure \ref{fig:err} and Figure \ref{fig:err2}. A clear pattern emerged: the prediction errors in neighboring areas tend to be similar. If the traffic demand in one area is underestimated, it is highly likely that the traffic demand in surrounding areas will also be underestimated, and vice versa. 

\begin{figure}[!h]
    \centering
    \includegraphics[width=0.95\linewidth]{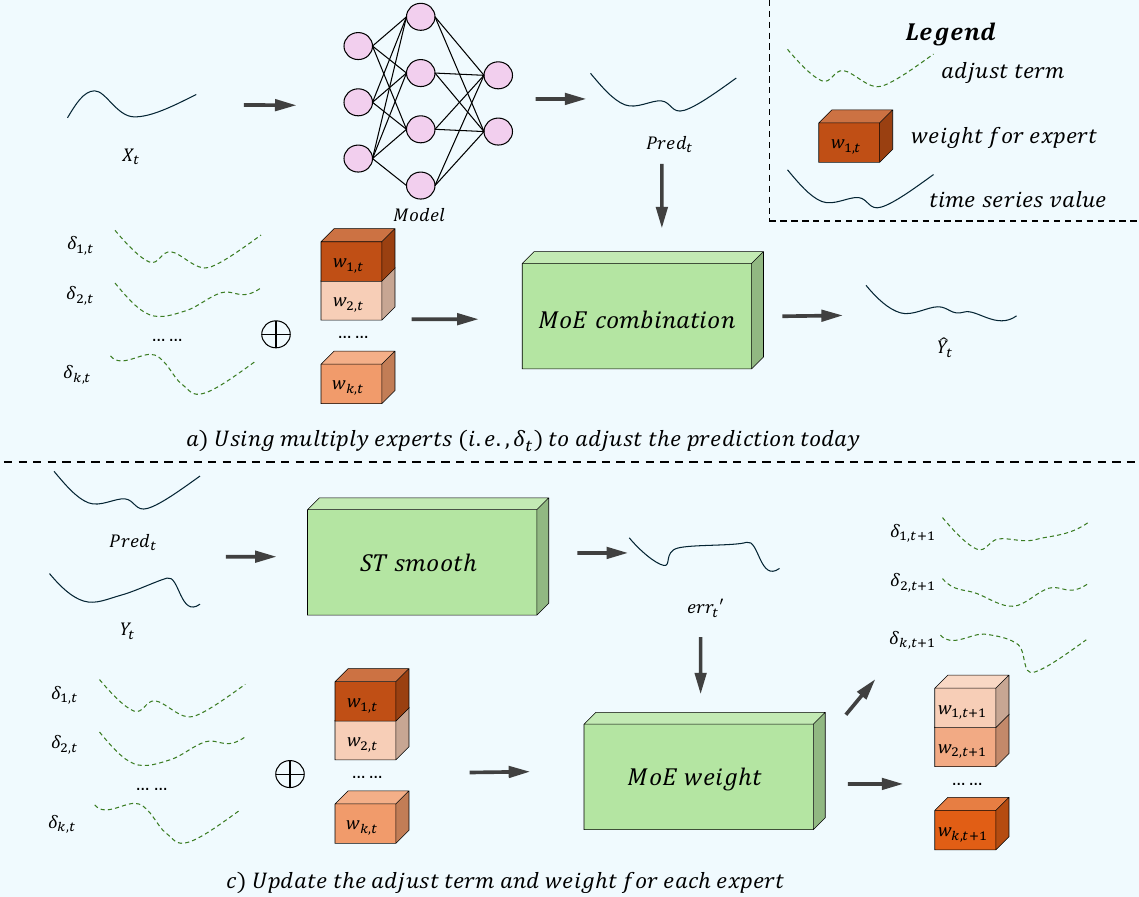}
    \caption{The workflow of our method}
    \label{fig:workflow}
\end{figure}
This suggests that the adjustments needed for different regions may not be independent but rather correlated. Since the adjustment amounts for adjacent areas are related, we can apply a spatial smoothing operation to obtain a more robust estimation of the required adjustments for each region.

We select the most basic graph convolution method for spatial smoothing. Specifically, for a region $i$, if its error is ${err}_i$, and it has $m$ neighbors with adjustment terms ${err}_1,\ldots,{err}_m$, then we update ${err}_i$ as:
\begin{equation}
    {err}_i^\prime=\left(1-\gamma\right){err}_i+\gamma\frac{\sum_{j=1}^{m}{err}_m}{m}
\end{equation}

Similarly, it can also be found that the adjustment values for different time slots are not independent. For example, if each time slot means one hour and the prediction at 6 AM tends to be overestimated, it is highly likely that the prediction at 7 AM will also be overestimated. Therefore, if we want to determine the adjustment for 7 AM, we can not only use the true value minus the predicted value from 7 AM of the previous day but also take a weighted average of errors from nearby time slots—such as 6 AM and 8 AM. Mathematically, this is defined as a one-dimensional convolution. We only need to define a convolution kernel and perform convolution along the time dimension.

The next question is how to determine the parameter $\gamma$ in spatial smoothing and the convolution kernel in temporal smoothing? In fact, both spatial and temporal smoothing are essentially linear transformations. These parameters are differentiable—that is, we can compute the loss after observing the true traffic demand, then back propagate it, and obtain gradients for both parameters. Therefore, in practical deployment, we propose an adaptive method to determine the smoothing parameters using online gradient descent. 

\subsection{The FORESEE Frame  work: Integration and Deployment}

Building upon the temporal correction mechanism and spatiotemporal smoothing modules, we integrate them into a unified online forecasting framework named \textbf{FORESEE} (\textbf{F}orecasting \textbf{O}nline with \textbf{R}esidual \textbf{S}moothing and \textbf{E}nsemble of \textbf{E}xperts). And we provide the pseudo code in following Algorithm \ref{alg:main} and Algorithm \ref{alg:ts_smooth}. The workflow of our method is shown in Figure \ref{fig:workflow}
\begin{algorithm}[!h]
\caption{FORESEE: Forecasting Online with Residual Smoothing and Ensemble Experts}
\label{alg:main}
\begin{algorithmic}[1]
\State \textbf{Initialize:}
\State \quad Set of smoothing factors $\mathcal{A} = \{\alpha_1, \alpha_2, \dots, \alpha_k\}$
\State \quad Expert weights: $w_{i,1} \gets \frac{1}{k}$ for $i = 1, \dots, k$
\State \quad Adjustment terms: $\delta_{i,1} \gets \mathbf{0}$ for $i = 1, \dots, k$ (shape $n \times 24\times2$)
\State \quad Graph $\mathcal{G}$ (adjacency matrix), spatial smooth parameter $\gamma$ ,temporal kernel $\mathbf{K}$

\State \quad Learning rates: $\eta$ (for experts), $\eta_\gamma$, $\eta_K$ (for OGD on smoothing parameters)

\For{each day $t = 1, 2, \dots$}
    \State Receive base model prediction $\mathbf{pred}_t$ (shape $n \times 24 \times2$)
    \For{each expert $i = 1$ to $k$}
        \State $\mathbf{\hat{y}}_{i,t} \gets \mathbf{pred}_t + \delta_{i,t}$ \Comment{Apply adjustment}
    \EndFor
    \State $\mathbf{\hat{y}}_t \gets \sum_{i=1}^k w_{i,t} \cdot \mathbf{\hat{y}}_{i,t}$ \Comment{Combine expert predictions}
    \State Observe true value $\mathbf{y}_t$
    \State $\mathbf{err}_t \gets \mathbf{y}_t - \mathbf{pred}_t$ \Comment{Compute raw prediction error}
    
    \State $\mathbf{err}_t^{\text{smooth}} \gets \text{TemporalSpatialSmooth}(\mathbf{err}_t, \mathcal{G}, \mathbf{K}, \gamma)$ \Comment{Alg. 2: Smooth the error}
    
    \For{each expert $i = 1$ to $k$}
        \State $\delta_{i,t+1} \gets \alpha_i \cdot \delta_{i,t} + (1 - \alpha_i) \cdot \mathbf{err}_t^{\text{smooth}}$ \Comment{Update adjustment using smoothed error}
    \EndFor

    \State \textbf{Update smoothing parameters via Online Gradient Descent:}
    \State Observe true value $y_t$ at the end of day $t$
    \State $\mathcal{L}_t \gets \left\| \mathbf{\hat{y}}_t - \mathbf{y}_t \right\|^2_2$ \Comment{Compute loss}
    \State $\gamma \gets \gamma - \eta_\gamma \cdot \frac{\partial \mathcal{L}_t}{\partial \gamma}$ \Comment{Update spatial parameter}
    \State $\mathbf{K} \gets \mathbf{K} - \eta_K \cdot \frac{\partial \mathcal{L}_t}{\partial \mathbf{K}}$ \Comment{Update temporal kernel parameters}

    \State \textbf{Update expert weights:}
    \For{each expert $i = 1$ to $k$}
        \State $L_{i,t} \gets \left\| \mathbf{\hat{y}}_{i,t} - \mathbf{y}_t \right\|^2_2$ \Comment{Squared error for expert $i$}
        \State $w_{i,t+1} \gets \frac{w_{i,t} \cdot \exp(-\eta \cdot L_{i,t})}{\sum_{j=1}^k w_{j,t} \cdot \exp(-\eta \cdot L_{j,t})}$ \Comment{update weights}
    \EndFor
\EndFor
\end{algorithmic}
\end{algorithm}

\begin{algorithm}[!h]
\caption{Spatiotemporal Smoothing of Error} 
\label{alg:ts_smooth}
\begin{algorithmic}[1]
\Function{TemporalSpatialSmooth}{$\mathbf{err}, \mathcal{G}, \mathbf{K}, \gamma$}
    \State \textbf{Input:} Raw error $\mathbf{err}$ ($n \times 24 \times 2$), graph $\mathcal{G}$, kernel $\mathbf{K}$, factor $\gamma$
    \State \textbf{Output:} Smoothed error $\mathbf{err}^{\text{smooth}}$ ($n \times 24 \times2$)
    \State Initialize $\mathbf{err}^{\text{spatial}} \gets \mathbf{0}_{n \times 24 \times 2}$
    \For{each region $i = 1$ to $n$}
        \State $\mathcal{N}(i) \gets \{j \mid \mathcal{G}_{ij} = 1\}$ \Comment{Get neighbors of region $i$}
        \State $m \gets |\mathcal{N}(i)|$
        \If{$m > 0$}
            \For{each hour $h = 1$ to $24$}
                \State $\mathbf{err}^{\text{spatial}}[i, h] \gets (1-\gamma) \cdot \mathbf{err}[i, h] + \gamma \cdot \frac{1}{m} \sum_{j \in \mathcal{N}(i)} \mathbf{err}[j, h]$
            \EndFor
        \Else
            \State $\mathbf{err}^{\text{spatial}}[i, :] \gets \mathbf{err}[i, :]$ \Comment{No neighbors, no spatial smoothing}
        \EndIf
    \EndFor
    \For{each region $i = 1$ to $n$}
        \State $\mathbf{err}^{\text{smooth}}[i, :] \gets \text{conv1d}(\mathbf{err}^{\text{spatial}}[i, :], \mathbf{K}, \text{mode='same'})$ \Comment{Apply temporal convolution}
    \EndFor
    \State \Return $\mathbf{err}^{\text{smooth}}$
\EndFunction
\end{algorithmic}
\end{algorithm}

\section{Theoretical analysis}
\label{sec:theo}
Since our method is relatively traditional—for example, using past errors to correct future predictions, applying EMA to smooth these corrections, and employing mixture-of-experts for hyperparameter adaptation—this implies that our approach possesses theoretically analyzable properties. Therefore, in this section, we provide theoretical guarantees for the use of past errors to correct future predictions, the EMA method, and the mixture-of-experts approach.

It should be noted that these theoretical insights are fairly straightforward for readers familiar with online prediction (such as online convex optimization). Additionally, although we use spatiotemporal smoothing to address spatiotemporal forecasting—and empirically observe its effectiveness—theoretical analysis for this would be complex. Thus, we omit this part.

\subsection{Data generation mechanism}

First, we make some assumptions about the data generation mechanism.

Suppose in each day $t$, the data is $y_t\in R^{24\times n\times2}$ where $n$ is the number of regions. $y_{t,i,j,0}$ means the inflow of the $j$-th region in the $i$-hour of the $t$-day. And we have a model $f$ to predict $y_t$. And $y_t$ is generated as:
\begin{equation}
    y_t=f(x_t)+g_t+\epsilon_t
\end{equation}

Where $x_t$ is the input of the model and $g_t$ is the bias term with the same shape of $y_t$, $\epsilon_t$ is independent random noise with mean 0 and covariate matrix $\Sigma_t$ of shape $R^{\left(24\times n\times2\right)\times(24\times n\times2)}$. Therefore, $g_t$ represents the mismatch between the prediction model and the actual data generation mechanism, while $\epsilon_t$ denotes the inherent uncertainty of the problem itself. And we assume $g_t$ follows a random walk process such that: $g_{t+1}=g_t+v_t$, $v_t$ is a random vector and independent in each time step with mean 0. We use the following example to illustrate the meaning of random work bias assumption. If the model's predicted value from the previous day is slightly larger than the actual value—for instance, larger by six—then on the next day, the model's prediction compared to the actual value might be larger by five to seven with high probability. In other words, it may fluctuate slightly upward or downward around the value of six, but overall, it generally remains close to six.

\subsection{Theorems}
\newtheorem{theorem}{Theorem} 
\begin{theorem}[The performance of original model]
If we let $\hat{y}_t=f(x_t)$, then the prediction error during a period of $T$ days will be:
$$
\mathbb{E}\sum_{t=1}^{T}\left|\left|\hat{y}-y_t\right|\right|_2^2=\mathbb{E}\sum_{t=1}^{T}\left|\left|g_t\right|\right|_2^2+\sum_{t=1}^{T}{Tr(\mathrm{\Sigma}_t)}
$$
\label{Tho:ori}
\end{theorem}
\begin{theorem}[The performance of using errors in the proceeding day to adjust the prediction in the next day]
    If we let $\hat{y}=f\left(x_t\right)+(y_{t-1}-f\left(x_{t-1}\right))$, Then the prediction error during a period of $T$ days will be:
    \begin{equation}
        \mathbb{E}\sum_{t=1}^{T}\left|\left|\hat{y}_t-y_t\right|\right|_2^2=\mathbb{E}\sum_{t=2}^{T}\left|\left|v_t\right|\right|_2^2+Tr\left(\mathrm{\Sigma}_1\right)+Tr(\mathrm{\Sigma}_T)+2\sum_{t=2}^{T-1}{Tr(\mathrm{\Sigma}_t)}
    \end{equation}
    For big $T$ we can omit $Tr\left(\mathrm{\Sigma}_1\right)+Tr(\mathrm{\Sigma}_t)$ and approximate these bound as:
    \begin{equation}
        \mathbb{E}\sum_{t=1}^{T}\left|\left|\hat{y}_t-y_t\right|\right|_2^2=\ \ \mathbb{E}\sum_{t=2}^{T}\left|\left|v_t\right|\right|_2^2+2\sum_{t=1}^{T}{Tr(\mathrm{\Sigma}_t)}
    \end{equation}
    \label{the:pro}
\end{theorem}
\textbf{Remark:} Comparing the results of Theorem \ref{Tho:ori} and Theorem \ref{the:pro}, we see a difference in the first term. In Theorem \ref{Tho:ori}, the first term is the sum of the squares of $g_t$. In Theorem \ref{the:pro}, the first term is the sum of the squares of the differences between $g_t$ and $g_{t-1}$, i.e, $v_t$.

We know that $g_t$ represents the overall bias, which means on day $t$, the average difference between the predicted value and the true value. In adjacent days, this value is often similar. For instance, during a pandemic, which reduces traffic demand, if we know the prediction is more the the true value by 10 on one day, then we can expect a similar error on the next day. As a result, the first term in Theorem \ref{the:pro} tends to be smaller than in Theorem \ref{Tho:ori}. For example, if the biases over three days are 10, 11, and 12, Theorem \ref{Tho:ori} gives an error of $10^2 + 11^2 + 12^2$, while Theorem \ref{the:pro} gives an error of $1^2 + 1^2$ (since the changes are 1 and 1), which is much smaller.

However, the second term in Theorem \ref{Tho:ori} is the sum of the random errors, while in Theorem \ref{the:pro}, it is twice the sum of the random errors. This is because in Theorem \ref{the:pro}, besides the inherent randomness, we use the error from the previous day to adjust the prediction on the next day, and that error contains noise, which increases the second term.
\begin{theorem}[The performance of using EMA errors in the proceeding days to adjust the prediction on the next day]
    If we let $\delta_t=\alpha\delta_{t-1}+(1-\alpha)\left(y_{t-1}-f\left(x_{t-1}\right)\right)$ and$\ \hat{y}_t=f\left(x_t\right)+\delta_t$, Then the prediction error during a period of $T$ days will approximate (for large $T$):
    \begin{equation}
        \mathbb{E}\sum_{t=1}^{T}||\hat{y}_t-y_t||^2_2=\frac{\mathbb{E}\sum_{t=1}^{T-1}\left|\left|v_{t+1}\right|\right|^2}{1-\alpha^2}+\frac{(2+\alpha)\sum_{t=1}^{T-1}{{Tr(\mathrm{\Sigma}}_t)}}{1+\alpha}
    \end{equation}
    \label{theo:ema}
\end{theorem}
\textbf{Remark:} We can compare this result with Theorem \ref{the:pro}. First, the first term changes from $\sum_{t=1}^{T-1}\left|\left|v_{t+1}\right|\right|^2$ to $\frac{\sum_{t=1}^{T-1}\left|\left|v_{t+1}\right|\right|^2}{1-\alpha^2}$. This term becomes larger. However, the second term becomes smaller because $\frac{(2+\alpha)\sum_{t=1}^{T-1}{{Tr(\Sigma}_t)}}{1+\alpha}$ is less than $2\sum_{t=1}^{T}{Tr(\Sigma_t)}$. This is a classic bias-variance trade-off. When we use exponential moving average to estimate the correction needed in each day, we introduce the errors from previous days, so the bias term (first term) increases. But because we incorporate more information from the prediction errors in more days, the variance is reduced, so the second term decreases. Thus, there will be an optimal $\alpha$ that balances bias and variance, resulting in the best prediction outcome.
\begin{theorem}[The performance of using multiple EMAs and mixture of experts]
    If we use multiple EMAs (each EMA can be called as an expert) with different $\alpha$s and use the following equation to update the weight for each expert:
    \begin{equation}
        w_{i,t+1}=\frac{w_{i,t}e^{-\eta\left({\hat{y}}_{i,t}-y_t\right)^2}}{\sum_{j}{w_{j,t}e^{-\eta\left({\hat{y}}_{j,t}-y_t\right)^2}}}
    \end{equation}
    And the final combined prediction $\hat{y}_t$ is:
    \begin{equation}
        {\hat{y}}_t=\sum_{i=1}^{k}{{\hat{y}}_{i,t}w_{i,t}}
    \end{equation}
    And let $L_i=\min_i{\frac{\sum_{t=1}^{T-1}\left|\left|v_{t+1}\right|\right|^2}{1-\alpha_i^2}+\frac{(2+\alpha_i)\sum_{t=1}^{T-1}{{Tr(\Sigma}_t)}}{1+\alpha_i}}$ be the test error with the best expert and $\left|\left|\hat{y}-y_t\right|\right|_2^2$ is bounded by $B$ in all $t$.
Then there exists a best $\eta$ such that:
\begin{equation}
    \mathbb{E}\sum_{t=1}^{T}\left(\hat{y}_t-y_t\right)^2\le L_i+\sqrt{\frac{1}{2}TB^2logk}
\end{equation}
\label{theo:emas}
\end{theorem}
\textbf{Remark:} The first term of the upper bound represents the error of the best expert, which balances the bias and variance terms using the most suitable $\alpha$. The second term $\sqrt{\frac{1}{2}TB^2logk}$ is the regret due to the mixture algorithm which grows sublinearly with $T$ and logarithmically with the number of experts $k$. $\eta$ can be tuned optimally if $T$ and $B$ is known in advance.

Besides, the first term grows linearly with $T$ because both $\frac{\sum_{t=1}^{T-1}\left|\left|g_t-g_{t+1}\right|\right|^2}{1-\alpha_i^2}$ and $\frac{(2+\alpha)\sum_{t=1}^{T-1}{{Tr(\Sigma}_t)}}{1+\alpha}$ are sums over $T-1$ terms, each term typically bounded, leading to linear growth in $T$. The second term grows only as the square root of $T$. As a result, for large $T$ (long-term deployment), the linear growth of the first term dominates the square root growth of the second term. Therefore, the mixture algorithm preferences almost the same as the best expert, and thus we nearly find the optimal $\alpha$ without prior knowledge. 

This explains why our method is robust - different datasets experience varying levels of traffic pattern changes during deployment. For datasets with minimal traffic pattern changes, the optimal $\alpha$ approaches 1, meaning we nearly make no adjustments. For datasets with significant traffic pattern changes, the optimal $\alpha$ approaches 0, requiring aggressive adjustments. Our mixture of experts strategy can automatically find the best smoothing coefficient $\alpha$, allowing adaptive decision-making between aggressive or conservative updates. Specifically, for datasets with almost no changes, the expert with $\alpha=1$ will dominate the weighting, making our final output nearly identical to the original model's predictions. Thus, the results remain comparable to directly using the original model. However, other online prediction methods require model updates during deployment using new data. These updates can cause overfitting. Therefore, on some datasets, existing online prediction methods perform even worse than simply deploying the original model without any updates.

\section{Experiments}
\subsection{Setup}
\subsubsection{Datasets}
To verify the broad effectiveness of our proposed method, we collected seven datasets. These include bike\footnote{https://citibikenyc.com/system-data} and taxi\footnote{https://www.nyc.gov/site/tlc/about/tlc-trip-record-data.page} usage data from New York City, bike\footnote{https://divvybikes.com/system-data} and taxi\footnote{https://data.cityofchicago.org/Transportation/Taxi-Trips-2013-2023-/wrvz-psew/about\_data} usage data from Chicago, bike usage data from the Bay Area\footnote{https://www.lyft.com/bikes/bay-wheels/system-data}, bike usage data from Toronto\footnote{https://open.toronto.ca/dataset/bike-share-toronto-ridership-data/}, and bike usage data from Boston\footnote{https://bluebikes.com/system-data}.  

For the taxi datasets, both New York and Chicago provide official zoning plans. We aggregated the taxi data based on these official divisions, calculating hourly pick-ups and drop-offs for each zone.  As for the bike datasets, the raw data consists of individual bike trip records, each including start time, end time, start location, and end location. We manually divided the five regions—New York, Chicago, Boston, Toronto, and the Bay Area—into grids. Then, we aggregated each bike trip record into hourly counts of bike pick-ups and drop-offs for each grid. 

\begin{figure}[!h]
    \centering
    \includegraphics[width=0.9\linewidth]{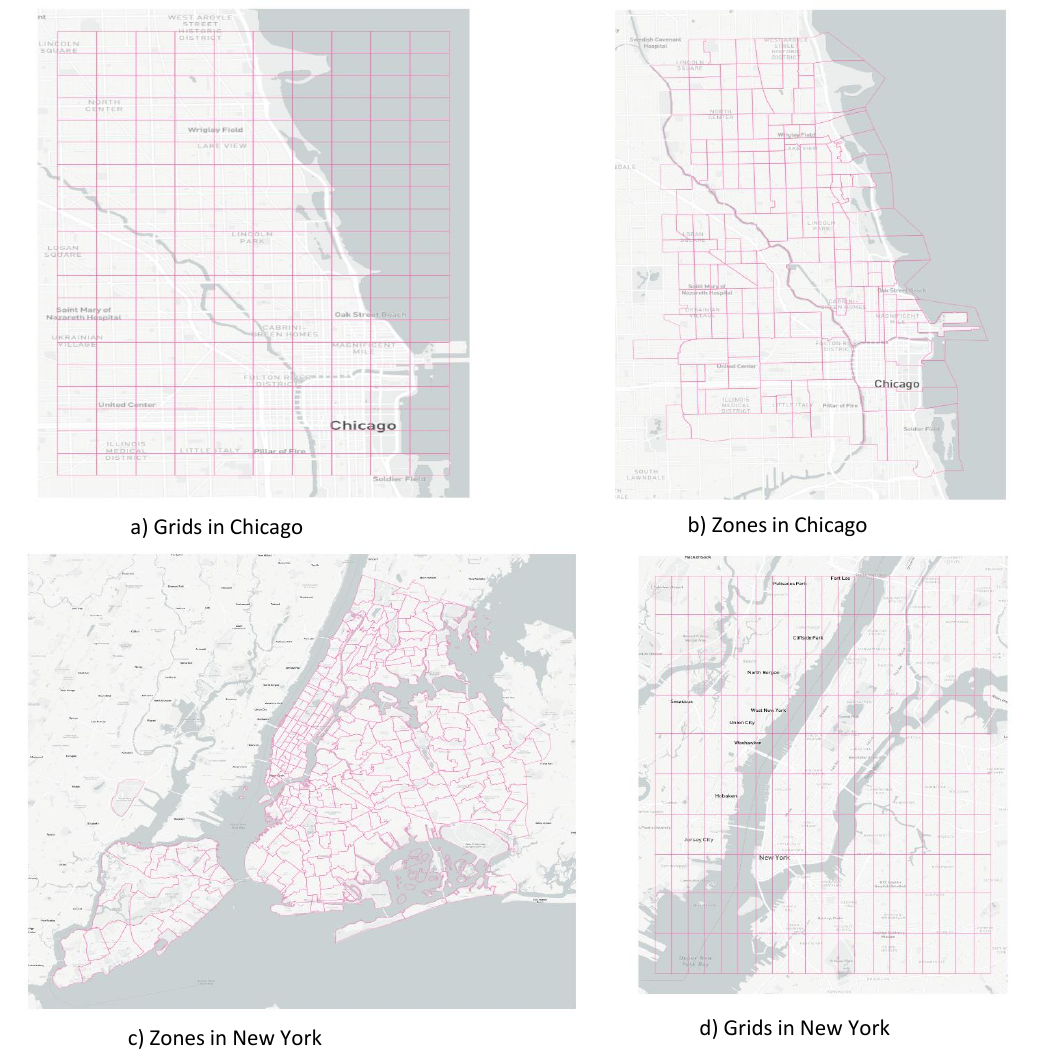}
    \caption{The zones in New York and Chicago}
    \label{fig:region}
\end{figure}

\begin{figure}[!h]
    \centering
    \includegraphics[width=0.9\linewidth]{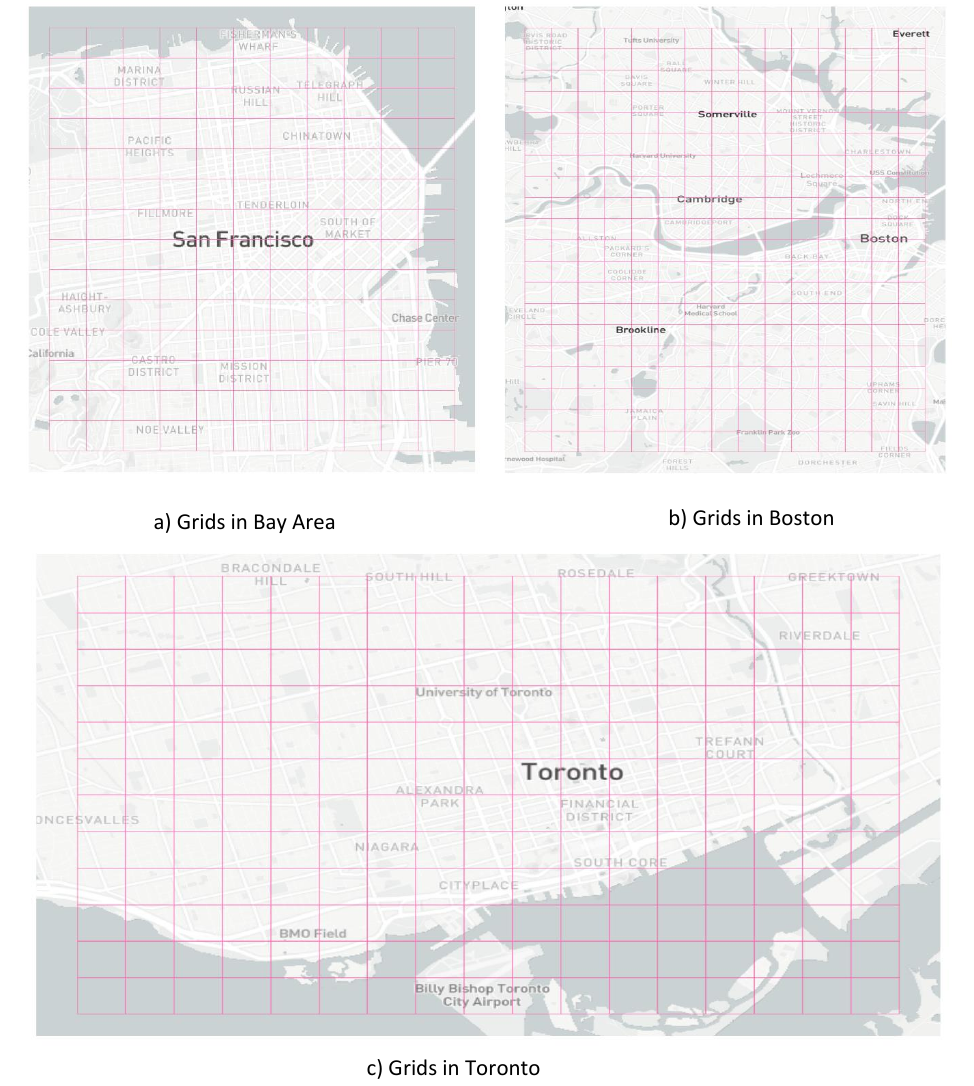}
    \caption{The grids in Bay Area, Boston and Toronto}
    \label{fig:grid}
\end{figure}

Additionally, we used data from January to November 2019 as the training set, data from December 2019 as the validation set, and deployed the model on data from January to June 2020. We chose this period because the COVID-19 pandemic occurred in 2020. It is evident that the data from January to June 2020 likely exhibits distribution shift compared to the training data. Therefore, online adaptation of the model is necessary. Finally we deleted the grids or zones with average taxi or bike usage less than 2 in the experiments.
\subsubsection{Prediction task and base models}
Our specific task is using the usage data of bikes or taxis over the past six hours to predict the usage for the next hour.

Regarding the prediction model, to demonstrate the applicability of our method across different model types, we selected three models: two traditional spatio-temporal neural network models—STGCN \citep{yu2018spatio} and GWNET \citep{wu2019graph}—and, considering the growing trend of applying large foundation models to traffic prediction, we also included a foundation model for spatio-temporal prediction called OpenCity \citep{li2024opencity}.

\subsubsection{Baselines}
For baseline methods, we found only one previously proposed method specifically designed for online adaptation in traffic prediction, namely ADCSD \citep{guo2025online}, so we included it in the comparisons. Additionally, several methods have been proposed for online adaptation of time series models, including OneNet \citep{wen2023onenet}, ELF \citep{leelightweight}, FSNet \citep{phamlearning}, DSOF \citep{lau2025fast}, and Proceed \citep{zhao2025proactive}. Moreover, two test time adaption methods for time series: TAFAS \citep{kim2025battling} and PETSA \citep{medeiros2025accurate} were also included as baselines. Finally, a basic baseline method—Online Gradient Descent (OGD) was also included.

It is worth noting that OneNet, FSNet, DSOF, and Proceed all involve retraining the model, while OGD requires back propagating gradients during online deployment to update all model parameters. These methods are computationally expensive. For traditional models, this computational cost may be acceptable, but for foundation models like OpenCity, which have a large number of parameters, updating the model online requires large GPU memory. Therefore, for OpenCity, we only compared ADCSD and ELF, which do not require gradient backpropagation or model retraining.
\subsubsection{Implemental details}
Our method involves several hyperparameters. We have summarized the hyperparameters used in our experiments in Table \ref{table:hyper}.
\begin{table}[!h]
\centering
\caption{Hyperparameters in our method}
\label{table:hyper}
\begin{tabular}{@{}cc@{}}
\toprule
Hyperparameters & Value  \\ \midrule
$\{\alpha_i\}$ & $\{0.7,0.8,0.9,1\}$   \\
  $\eta$ & 10 \\ 
 $\eta_k,\eta_\gamma$ &0.01\\ \bottomrule
\end{tabular}
\end{table}
We use MAE (mean absolute err) and RMSE (rooted mean square error) on test dataset to evaluate the performance of different online method. These two metrics are defined as:
\begin{equation}
    MAE=\frac{1}{n}\sum y_{pred_i}-y_{actual_i}|
\end{equation}
\begin{equation}
    RMSE=\sqrt{\frac{1}{n}\left(y_{pred_i-}y_{actual_i}\right)^2}
\end{equation}
Where $y_{pred_i}$ is the prediction value for the $i$-th sample, and the $y_{actual_i}$ is the actual value for the $i$-th sample.
\subsection{Main results}
We provide the results from seven datasets and three prediction models in Table \ref{table:result1} to Table \ref{table:result3}. The last rows in these tables show the percentage error reduction of our method compared to the the base model. The best result is colored in red and the second best result is colored in blue with an underline. 
\begin{table}[!h]
\centering
\caption{Results of NTCBIKE and NYCTAXI datasets}
\label{table:result1}
\fontsize{7}{6}\selectfont
\renewcommand{\arraystretch}{1.5}
\begin{tabular}{@{}c|cccccc|cccccc@{}}
\toprule
Dataset              & \multicolumn{6}{c|}{\textbf{NTCBIKE}}                                                                                                                                                                                                                                              & \multicolumn{6}{c}{\textbf{NYCTAXI}}                                                                                                                                                                                                                                                \\ \midrule
Model                & \multicolumn{2}{c}{STGCN}                                                                & \multicolumn{2}{c}{GWNET}                                                                 & \multicolumn{2}{c|}{Opencity}                                                               & \multicolumn{2}{c}{STGCN}                                                                 & \multicolumn{2}{c}{GWNET}                                                                  & \multicolumn{2}{c}{Opencity}                                                               \\ \hline
Metrics & RMSE                                        & MAE                                        & RMSE                                        & MAE                                         & RMSE                                         & MAE                                         & RMSE                                         & MAE                                        & RMSE                                         & MAE                                         & RMSE                                         & MAE                                         \\
Ori                  & 11.043                                      & 5.142                                      & 10.492                                      & 5.168                                       & 12.909                                       & 6.912                                       & 17.188                                       & 6.713                                      & 18.623                                       & 6.575                                       & 15.964                                       & 5.763                                       \\
FSNet                & 9.716                                       & 4.88                                       & {\color[HTML]{0070C0} {\ul \textbf{9.674}}} & {\color[HTML]{0070C0} {\ul \textbf{4.793}}} & -                                            & -                                           & 11.478                                       & 4.667                                      & 11.635                                       & 4.563                                       & -                                            & -                                           \\
Onenet               & 10.593                                      & 5.264                                      & -                                           & -                                           & -                                            & -                                           & 11.618                                       & 4.879                                      & -                                            & -                                           & -                                            & -                                           \\
ADCSD                & 9.439                                       & 4.843                                      & 9.995                                       & 5.126                                       & 12.732                                       & 6.862                                       & {\color[HTML]{0070C0} {\ul \textbf{10.756}}} & {\color[HTML]{0070C0} {\ul \textbf{4.31}}} & {\color[HTML]{0070C0} {\ul \textbf{11.138}}} & {\color[HTML]{0070C0} {\ul \textbf{4.323}}} & {\color[HTML]{0070C0} {\ul \textbf{15.386}}} & {\color[HTML]{0070C0} {\ul \textbf{5.563}}} \\
OGD                  & 9.83                                        & 4.942                                      & 9.93                                        & 5.012                                       & -                                            & -                                           & 11.148                                       & 4.855                                      & 11.041                                       & 4.452                                       & -                                            & -                                           \\
ELF                  & 10.792                                      & 5.745                                      & 11.212                                      & 5.896                                       & {\color[HTML]{0070C0} {\ul \textbf{12.696}}} & {\color[HTML]{0070C0} {\ul \textbf{6.832}}} & 13.636                                       & 5.923                                      & 13.697                                       & 5.456                                       & 15.937                                       & 5.667                                       \\
DSOF                 & {\color[HTML]{0070C0} {\ul \textbf{9.293}}} & {\color[HTML]{0070C0} {\ul \textbf{4.45}}} & 11.147                                      & 5.064                                       & -                                            & -                                           & 11.103                                       & 5.121                                      & 16.948                                       & 6.006                                       & -                                            & -                                           \\
Proceed              & 9.638                                       & 4.874                                      & 10.736                                      & 5.335                                       & -                                            & -                                           & 10.896                                       & 4.813                                      & 11.3                                         & 4.463                                       & -                                            & -                                           \\
PETSA                & 10.486                                      & 5.152                                      & 12.248                                      & 5.941                                       &     -                                         &              -                               & 13.719                                       & 4.981                                      & 14.598                                       & 5.105                                       &         -                                     &                     -                        \\
TAFAS                & 10.561                                      & 5.113                                      & 12.665                                      & 6.032                                       & -          & -          & 13.681                                       & 4.914                                      & 14.697                                       & \multicolumn{1}{l}{5.167}                   & -     &-      \\
Foresee              & {\color[HTML]{FF0000} \textbf{9.148}}       & {\color[HTML]{FF0000} \textbf{4.429}}      & {\color[HTML]{FF0000} \textbf{9.656}}       & {\color[HTML]{FF0000} \textbf{4.754}}       & {\color[HTML]{FF0000} \textbf{12.29}}        & {\color[HTML]{FF0000} \textbf{6.475}}       & {\color[HTML]{FF0000} \textbf{10.381}}       & {\color[HTML]{FF0000} \textbf{3.909}}      & {\color[HTML]{FF0000} \textbf{10.567}}       & {\color[HTML]{FF0000} \textbf{3.933}}       & {\color[HTML]{FF0000} \textbf{14.704}}       & {\color[HTML]{FF0000} \textbf{5.306}}       \\
improve              & 17.20\%                                     & 13.90\%                                    & 8.00\%                                      & 8.00\%                                      & 4.70\%                                       & 6.30\%                                      & 21.30\%                                      & 41.80\%                                    & 43.30\%                                      & 40.20\%                                     & 7.90\%                                       & 7.90\%                                      \\ \bottomrule
\end{tabular}
\end{table}
\begin{table}[!h]
\centering
\caption{Results of CHIBIKE and CHITAXI datasets}
\label{table:result2}
\fontsize{7}{6}\selectfont
\renewcommand{\arraystretch}{1.5}
\begin{tabular}{@{}c|cccccc|cccccc@{}}
\toprule
Dataset              & \multicolumn{6}{c|}{\textbf{CHIBIKE}}                                                                                                                                                                                                                                              & \multicolumn{6}{c}{\textbf{CHITAXI}}                                                                                                                                                                                                                                              \\ \midrule
Model                & \multicolumn{2}{c}{STGCN}                                                                 & \multicolumn{2}{c}{GWNET}                                                                 & \multicolumn{2}{c|}{Opencity}                                                              & \multicolumn{2}{c}{STGCN}                                                                 & \multicolumn{2}{c}{GWNET}                                                                 & \multicolumn{2}{c}{Opencity}                                                              \\
Metrics & RMSE                                        & MAE                                         & RMSE                                        & MAE                                         & RMSE                                        & MAE                                         & RMSE                                        & MAE                                         & RMSE                                        & MAE                                         & RMSE                                        & MAE                                         \\
Ori                  & 2.605                                       & 1.337                                       & 2.527                                       & {\color[HTML]{0070C0} {\ul \textbf{1.309}}} & 2.826                                       & 1.706                                       & 3.389                                       & 1.386                                       & 3.449                                       & 1.362                                       & 3.832                                       & 1.46                                        \\
FSNet                & 2.549                                       & 1.353                                       & 2.498                                       & 1.332                                       & -                                           & -                                           & 3.369                                       & 1.258                                       & 3.363                                       & 1.243                                       & -                                           & -                                           \\
Onenet               & 2.523                                       & 1.444                                       & -                                           & -                                           & -                                           & -                                           & 3.318                                       & 1.746                                       & -                                           & -                                           & -                                           & -                                           \\
ADCSD                & {\color[HTML]{0070C0} {\ul \textbf{2.451}}} & 1.383                                       & 2.383                                       & 1.349                                       & {\color[HTML]{0070C0} {\ul \textbf{2.801}}} & {\color[HTML]{0070C0} {\ul \textbf{1.689}}} & {\color[HTML]{0070C0} {\ul \textbf{3.297}}} & 1.205                                       & 3.283                                       & 1.168                                       & {\color[HTML]{0070C0} {\ul \textbf{3.743}}} & {\color[HTML]{0070C0} {\ul \textbf{1.414}}} \\
OGD                  & 2.528                                       & 1.402                                       & {\color[HTML]{0070C0} {\ul \textbf{2.376}}} & 1.326                                       & -                                           & -                                           & 3.277                                       & 1.249                                       & {\color[HTML]{0070C0} {\ul \textbf{3.234}}} & {\color[HTML]{0070C0} {\ul \textbf{1.183}}} & -                                           & -                                           \\
ELF                  & 2.746                                       & 1.551                                       & 2.786                                       & 1.563                                       & 2.965                                       & 1.763                                       & 4.009                                       & 1.538                                       & 4.053                                       & 1.539                                       & 4.333                                       & 1.527                                       \\
DSOF                 & 2.779                                       & 1.723                                       & 3.057                                       & 1.784                                       & -                                           & -                                           & 4.985                                       & 3.524                                       & 6.491                                       & 4.254                                       & -                                           & -                                           \\
Proceed              & 2.452                                       & 1.361                                       & 2.391                                       & 1.343                                       & -                                           & -                                           & 3.452                                       & 1.36                                        & 3.283                                       & 1.269                                       & -                                           & -                                           \\
PETSA                & 2.493                                       & 1.344                                       & 2.662                                       & 1.416                                       &    -                                         &           -                                  & 3.53                                        & 1.172                                       & 4.144                                       & 1.283                                       &         -                                    &                       -                      \\
TAFAS                & 2.492                                       & {\color[HTML]{0070C0} {\ul \textbf{1.332}}} & 2.496                                       & 1.316                                       & -         &-         & 3.565                                       & {\color[HTML]{0070C0} {\ul \textbf{1.174}}} & 4.116                                       & 1.278                                       & -         & -         \\
Foresee              & {\color[HTML]{FF0000} \textbf{2.435}}       & {\color[HTML]{FF0000} \textbf{1.298}}       & {\color[HTML]{FF0000} \textbf{2.336}}       & {\color[HTML]{FF0000} \textbf{1.268}}       & {\color[HTML]{FF0000} \textbf{2.793}}       & {\color[HTML]{FF0000} \textbf{1.663}}       & {\color[HTML]{FF0000} \textbf{3.23}}        & {\color[HTML]{FF0000} \textbf{1.194}}       & {\color[HTML]{FF0000} \textbf{3.216}}       & {\color[HTML]{FF0000} \textbf{1.161}}       & {\color[HTML]{FF0000} \textbf{3.693}}       & {\color[HTML]{FF0000} \textbf{1.413}}       \\
improve              & 6.50\%                                      & 3.00\%                                      & 7.50\%                                      & 3.20\%                                      & 1.10\%                                      & 2.50\%                                      & 4.70\%                                      & 13.80\%                                     & 6.80\%                                      & 14.80\%                                     & 3.60\%                                      & 3.20\%                                      \\ \bottomrule
\end{tabular}
\end{table}

\begin{table}[!h]
\centering
\caption{Results of BOSBIKE, NACBIKE and TORBIKE datasets}
\label{table:result3}
\fontsize{6.5}{6}\selectfont
\renewcommand{\arraystretch}{1.5}
\setlength{\tabcolsep}{1.3pt}
\begin{tabular}{@{}c|cccccc|cccccc|cccccc@{}}
\toprule
Dataset              & \multicolumn{6}{c|}{\textbf{BOSBIKE}}                                                                                                                                                                                                                                            & \multicolumn{6}{c|}{\textbf{BAYBIKE}}                                                                                                                                                                                                                                              & \multicolumn{6}{c}{\textbf{TORBIKE}}                                                                                                                                                                                                                                             \\ \hline
Model                & \multicolumn{2}{c}{STGCN}                                                                & \multicolumn{2}{c}{GWNET}                                                                 & \multicolumn{2}{c|}{Opencity}                                                             & \multicolumn{2}{c}{STGCN}                                                                 & \multicolumn{2}{c}{GWNET}                                                                 & \multicolumn{2}{c|}{Opencity}                                                              & \multicolumn{2}{c}{STGCN}                                                                 & \multicolumn{2}{c}{GWNET}                                                                 & \multicolumn{2}{c}{Opencity}                                                             \\ \hline
Metrics & RMSE                                       & MAE                                         & RMSE                                        & MAE                                         & RMSE                                        & MAE                                        & RMSE                                        & MAE                                         & RMSE                                        & RMSE                                        & MAE                                         & RMSE                                        & RMSE                                        & MAE                                         & RMSE                                        & RMSE                                        & MAE                                         & RMSE                                       \\
Ori                  & 2.077                                      & 1.154                                       & 2.033                                       & 1.135                                       & 2.244                                       & 1.384                                      & 3.22                                        & 1.552                                       & 3.09                                        & 1.506                                       & {\color[HTML]{0070C0} {\ul \textbf{2.865}}} & {\color[HTML]{0070C0} {\ul \textbf{1.648}}} & 2.206                                       & 1.182                                       & 2.146                                       & 1.171                                       & 2.355                                       & 1.436                                      \\
FSNet                & 2.024                                      & 1.158                                       & 2.001                                       & 1.131                                       & -                                           & -                                          & 2.913                                       & 1.503                                       & 2.862                                       & 1.493                                       & -                                           & -                                           & 2.128                                       & 1.189                                       & 2.123                                       & 1.163                                       & -                                           & -                                          \\
Onenet               & 2.028                                      & 1.220                                        & -                                      & -                                         & -                                           & -                                          & 2.796                                       & 1.557                                       & -                                           & -                                           & -                                           & -                                           & 2.114                                       & 1.253                                       & -                                           & -                                           &                                             &                                            \\
ADCSD                & {\color[HTML]{0070C0} {\ul \textbf{1.97}}} & 1.186                                       & 1.948                                       & 1.177                                       & {\color[HTML]{0070C0} {\ul \textbf{2.185}}} & {\color[HTML]{0070C0} {\ul \textbf{1.35}}} & 2.669                                       & 1.496                                       & 2.783                                       & 1.514                                       & 3.043                                       & 1.718                                       & {\color[HTML]{0070C0} {\ul \textbf{2.103}}} & 1.246                                       & 2.103                                       & 1.23                                        & {\color[HTML]{0070C0} {\ul \textbf{2.322}}} & {\color[HTML]{0070C0} {\ul \textbf{1.41}}} \\
OGD                  & 1.996                                      & 1.189                                       & {\color[HTML]{0070C0} {\ul \textbf{1.944}}} & 1.176                                       & -                                           & -                                          & 2.701                                       & 1.498                                       & 2.76                                        & 1.444                                       & -                                           & -                                           & {\color[HTML]{0070C0} {\ul \textbf{2.103}}} & 1.234                                       & {\color[HTML]{0070C0} {\ul \textbf{2.082}}} & 1.218                                       &         -                                    &                              -              \\
ELF                  & 2.184                                      & 1.288                                       & 2.128                                       & 1.269                                       & 2.365                                       & 1.442                                      & 3.365                                       & 1.7                                         & 3.292                                       & 1.673                                       & 3.078                                       & 1.737                                       & 2.273                                       & 1.314                                       & 2.231                                       & 1.31                                        & 2.46                                        & 1.492                                      \\
DSOF                 & 2.184                                      & 1.288                                       & 2.357                                       & 1.487                                       & -                                           & -                                          & 3.397                                       & 1.805                                       & 3.84                                        & 2.09                                        & -                                           & -                                           & 2.291                                       & 1.374                                       & 2.459                                       & 1.451                                       & -                                           & -                                          \\
Proceed              & 1.982                                      & 1.18                                        & 2.102                                       & 1.2401                                      & -                                           & -                                          & 2.727                                       & 1.458                                       & 4.096                                       & 1.93                                        & -                                           & -                                           & 2.109                                       & 1.234                                       & 2.377                                       & 1.383                                       &                -                             &        -                                    \\
PETSA                & 1.996                                      & 1.143                                       & 1.986                                       & 1.122                                       &         -                                    &             -                               & {\color[HTML]{FF0000} \textbf{2.634}}       & {\color[HTML]{FF0000} \textbf{1.426}}       & {\color[HTML]{FF0000} \textbf{2.586}}       & {\color[HTML]{FF0000} \textbf{1.385}}       &           -                                  &                                 -            & 2.104                                       & 1.163                                       & 2.104                                       & 1.166                                       &              -                               &                                     -       \\
TAFAS                & 2.01                                       & {\color[HTML]{0070C0} {\ul \textbf{1.138}}} & 1.981                                       & {\color[HTML]{0070C0} {\ul \textbf{1.117}}} & -           & -     & 2.685                                       & 1.447                                       & 2.800                                         & 1.482                                       & -    & -   & 2.115                                       & {\color[HTML]{0070C0} {\ul \textbf{1.162}}} & 2.119                                       & {\color[HTML]{0070C0} {\ul \textbf{1.158}}} & -      &-       \\
Foresee              & {\color[HTML]{FF0000} \textbf{1.942}}      & {\color[HTML]{FF0000} \textbf{1.126}}       & {\color[HTML]{FF0000} \textbf{1.938}}       & {\color[HTML]{FF0000} \textbf{1.116}}       & {\color[HTML]{FF0000} \textbf{2.196}}       & {\color[HTML]{FF0000} \textbf{1.348}}      & {\color[HTML]{0070C0} {\ul \textbf{2.679}}} & {\color[HTML]{0070C0} {\ul \textbf{1.434}}} & {\color[HTML]{0070C0} {\ul \textbf{2.702}}} & {\color[HTML]{0070C0} {\ul \textbf{1.435}}} & {\color[HTML]{FF0000} \textbf{2.818}}       & {\color[HTML]{FF0000} \textbf{1.603}}       & {\color[HTML]{FF0000} \textbf{2.029}}       & {\color[HTML]{FF0000} \textbf{1.142}}       & {\color[HTML]{FF0000} \textbf{2.078}}       & {\color[HTML]{FF0000} \textbf{1.157}}       & {\color[HTML]{FF0000} \textbf{2.305}}       & {\color[HTML]{FF0000} \textbf{1.396}}      \\
improve              & 6.50\%                                     & 2.40\%                                      & 4.60\%                                      & 1.50\%                                      & 2.10\%                                      & 2.60\%                                     & 16.70\%                                     & 7.50\%                                      & 12.50\%                                     & 4.60\%                                      & 1.60\%                                      & 2.70\%                                      & 8.00\%                                      & 3.40\%                                      & 3.20\%                                      & 1.20\%                                      & 2.10\%                                      & 2.70\%                              \\ \bottomrule       
\end{tabular}
\end{table}

We can conclude our method as accurate and robust. For accuracy, it is clear that our model achieves the best prediction accuracy on most datasets and prediction models. Compared to the original method without online fine-tuning, our approach shows significant improvement, and the extent varies by dataset. For example, on the New York taxi dataset, accuracy improves by nearly 40\%, while on the New York bike dataset, it improves by about 10\%. On the Boston bike dataset, accuracy increases by approximately 5\%. This variation is understandable because traffic patterns change differently across datasets—some show more noticeable shifts, while others do not. Therefore, the improvement differs depending on the situation.

As for robustness, our method consistently improves accuracy compared to the original model in all cases. In contrast, other online prediction methods sometimes even lead to performance degradation. This happens because in scenarios with little distribution change, forcing online adaptation may cause overfitting, resulting in worse performance than without online adaptation. However, our method uses a multi-expert ensemble approach. One expert has a smoothing coefficient of one, meaning it acts as the original model without online adaptation. If there is indeed little distribution drift, the weight of this expert gradually increases. As a result, our method converges to the original model, ensuring that even when the distribution drift is minimal, the performance of our method remains robust. 

We plot several graphs showing the true values and predicted values in Figure \ref{fig:Visualization}

\begin{figure}[!h]
    \centering
    \includegraphics[width=0.9\linewidth]{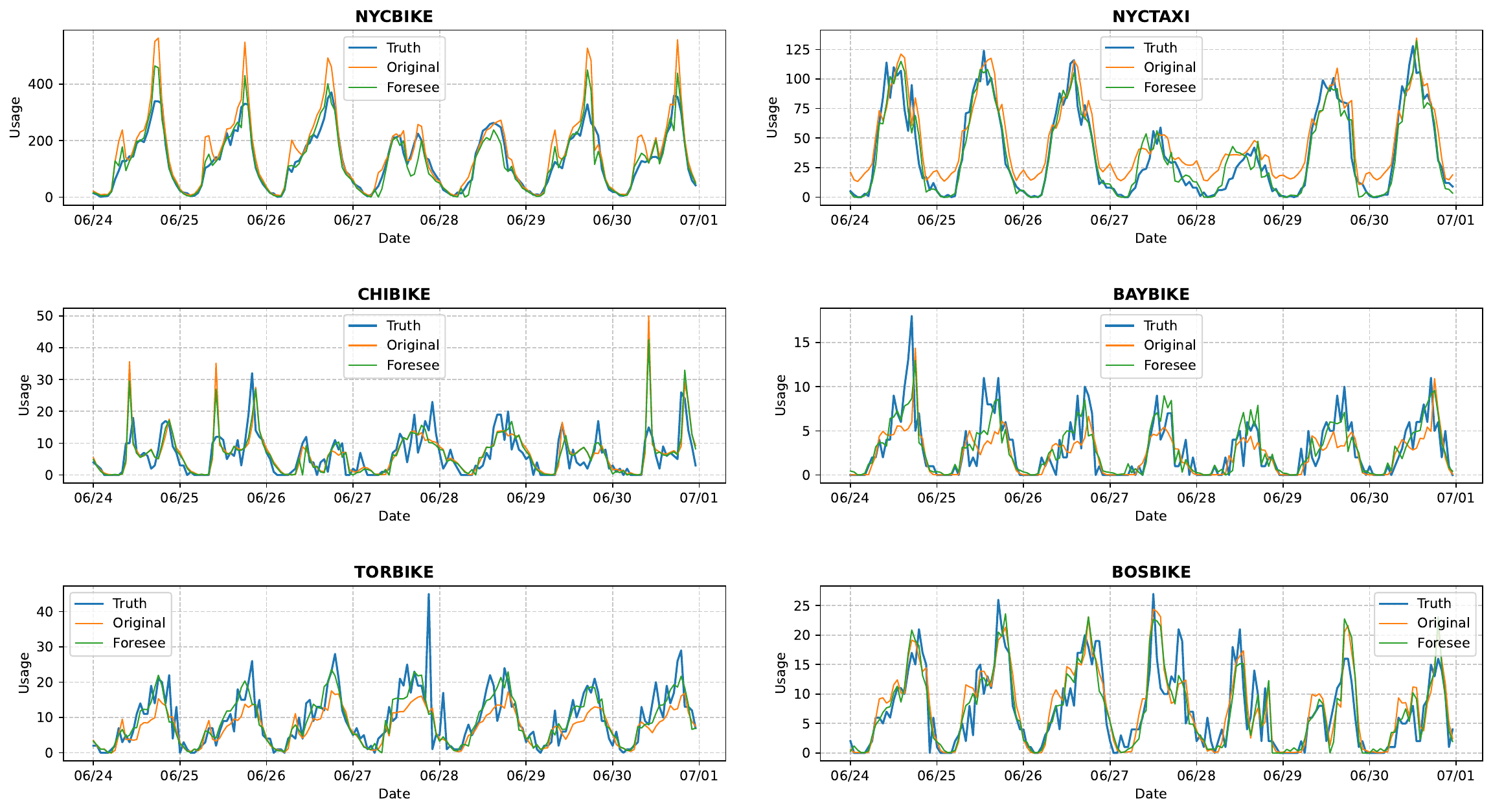}
    \caption{Visualization of our method}
    \label{fig:Visualization}
\end{figure}
\subsection{Method analysis}
\subsubsection{Analysis of $\alpha$}
In our experiments, one very important hyperparameter is $\alpha$, which represents the smoothing coefficient for errors across different days. In the main experimental section, we used a mixture of experts to approximate the optimal $\alpha$. Therefore, in this subsection, we first aim to conduct a sensitivity analysis on this $\alpha$ parameter.

Using the STGCN model as an example, we carried out additional experiments on four datasets: bike and taxi data from New York and Chicago. Specifically, instead of employing the mixture of experts approach, we tested $\alpha$ values from 0 to 1 in increments of 0.05. For each $\alpha$ value, we observed the resulting MAE after online deployment. The results are summarized in Figure \ref{fig:ema} below.

\begin{figure}[!h]
    \centering
    \includegraphics[width=0.9\linewidth]{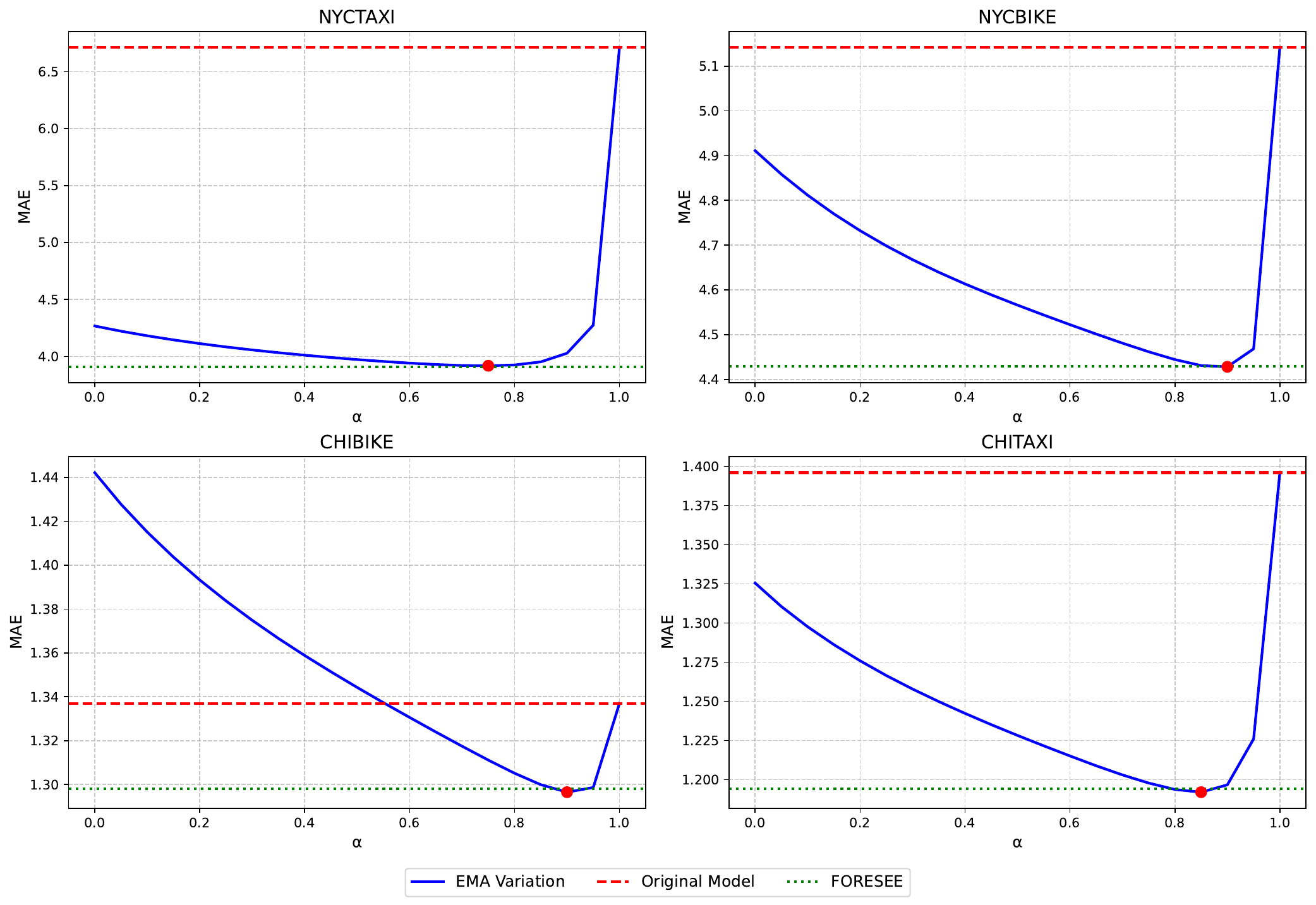}
    \caption{Results of using different $\alpha$}
    \label{fig:ema}
\end{figure}

From Figure \ref{fig:ema}, it can be observed that using different smoothing factor $\alpha$ leads to significant different MAE. Besides, the optimal $\alpha$ value is different for each dataset. It should also be emphasized that when $\alpha$ is set to zero, it essentially means using the previous day's error directly to adjust the next day's prediction. This simple approach actually performs quite poorly. Particularly noticeable is that on the Chicago bike dataset, this method yields even worse results than making no corrections at all.

However, when $\alpha$ is tuned to an appropriate value, much better results can be achieved. The red dots in the figure indicate the optimal $\alpha$ values. The green line represents the final results provided by our method. It can be observed that our method's performance is very close to that achieved with the optimal $\alpha$.

Moreover, Figure \ref{fig:ema} reveals a clear U-shaped trend in the relationship between MAE and $\alpha$. This pattern aligns well with our theoretical analysis in Section 4. Specifically, a very small $\alpha$ (e.g., $\alpha$=0) makes the correction term overly sensitive to daily noise, leading to unstable and inaccurate adjustments. Conversely, a very large $\alpha$ (close to 1) results in excessive smoothing, causing the method to react too slowly to genuine distribution shifts. Therefore, an optimal $\alpha$ exists that balances responsiveness to real changes and robustness against random fluctuations. The observed U-shaped curves across all datasets empirically validate this trade-off and further justify our design choice of using a mixture-of-experts mechanism to adaptively approximate the optimal $\alpha$ in practice.
\subsubsection{Analysis of spatial temporal smooth}

To evaluate the necessity and contribution of our method's spatiotemporal smoothing component, we conducted ablation studies We used the STGCN and GWNET models and conducted ablation experiments on all datasets. The results are shown in Table \ref{table:ab} below. The table compares three approaches: The original model without any online adaptation, Foresee method, Foresee method without the spatiotemporal smoothing module (Foresee-)

\begin{table}[!h]
\centering
\caption{Results of ablation experiments}
\label{table:ab}
\fontsize{7}{6}\selectfont
\renewcommand{\arraystretch}{1.5}\textbf{}
\begin{tabular}{@{}c|cccccc|cccccc@{}}
\toprule
Model   & \multicolumn{6}{c|}{STGCN}                                                                                       & \multicolumn{6}{c}{GWNET}                                                                                  \\ \midrule
Method  & \multicolumn{2}{c}{Ori} & \multicolumn{2}{c}{Foresee} & \multicolumn{2}{c|}{Foresee-}                            & \multicolumn{2}{c}{Ori} & \multicolumn{2}{c}{Foresee} & \multicolumn{2}{c}{Foresee-}                       \\
Dataset & RMSE        & MAE       & RMSE          & MAE         & RMSE                       & MAE                        & RMSE        & MAE       & RMSE          & MAE         & \multicolumn{1}{c}{RMSE} & \multicolumn{1}{c}{MAE} \\
NYCBIKE & 11.0430     & 5.1420    & 9.1480        & 4.4290      & 9.1960                     & 4.4460                     & 10.4920     & 5.1680    & 9.6560        & 4.7540      & 9.6113                   & 4.7672                  \\
NYCTAXI & 13.1880     & 6.7130    & 10.3810       & 3.9090      & 10.4440                    & 3.9900                     & 18.6230     & 6.5750    & 10.5670       & 3.9330      & 10.6139                  & 3.9848                  \\
CHIBIKE & 2.6050      & 1.3370    & 2.4350        & 1.2980      & 2.5100                     & 1.3120                     & 2.5270      & 1.3090    & 2.3360        & 1.2680      & 2.4128                   & 1.2809                  \\
CHITAXI & 3.3890      & 1.3860    & 3.2300        & 1.1940      & 3.2460                     & 1.2070                     & 3.4490      & 1.3620    & 3.2160        & 1.1610      & 3.2068                   & 1.1661                  \\
BOSBIKE & 2.0770      & 1.1540    & 1.9420        & 1.1260      & 2.0754                     & 1.1476 & 2.0330      & 1.1350    & 1.9380        & 1.1160      & 2.1188                   & 1.1646                  \\
BAYBIKE & 3.2200      & 1.5520    & 2.6790        & 1.4340      & 2.7095 &1.4636 & 3.0900      & 1.5060    & 2.7020        & 1.4350      & 2.7342                   & 1.4569                  \\
ROEBIKE & 2.2060      & 1.1820    & 2.0290        & 1.1420      & 2.0514 & 1.1664 & 2.1460      & 1.1710    & 2.0780        & 1.1570      & 2.0900                   & 1.1601                  \\ \bottomrule
\end{tabular}
\end{table}



The results show that removing the spatiotemporal smoothing module increases test error. For example, when using STGCN on the NYC bike dataset, RMSE increased from 9.148 to 9.196, and MAE rose from 4.429 to 4.446. Besides, it's worth noting that even without the smoothing module, our method still significantly outperforms the original model.
\subsubsection{Computational cost comparison}

Ideally, an online adaptation algorithm should adjust a given model with minimal cost to better fit the test data. Theoretically, our method does not require gradient computation on the model's parameters, whereas many baseline methods do. This means our approach should use less GPU memory and require less computation. To verify this, we compared the GPU memory usage and the total runtime (in RTX 4090D and i5-12400F) between our method and all baseline methods using the GWNET model on bike and taxi datasets from New York. The results are shown in Figure \ref{fig:eff}. 
\begin{figure}[!h]
    \centering
    \includegraphics[width=0.9\linewidth]{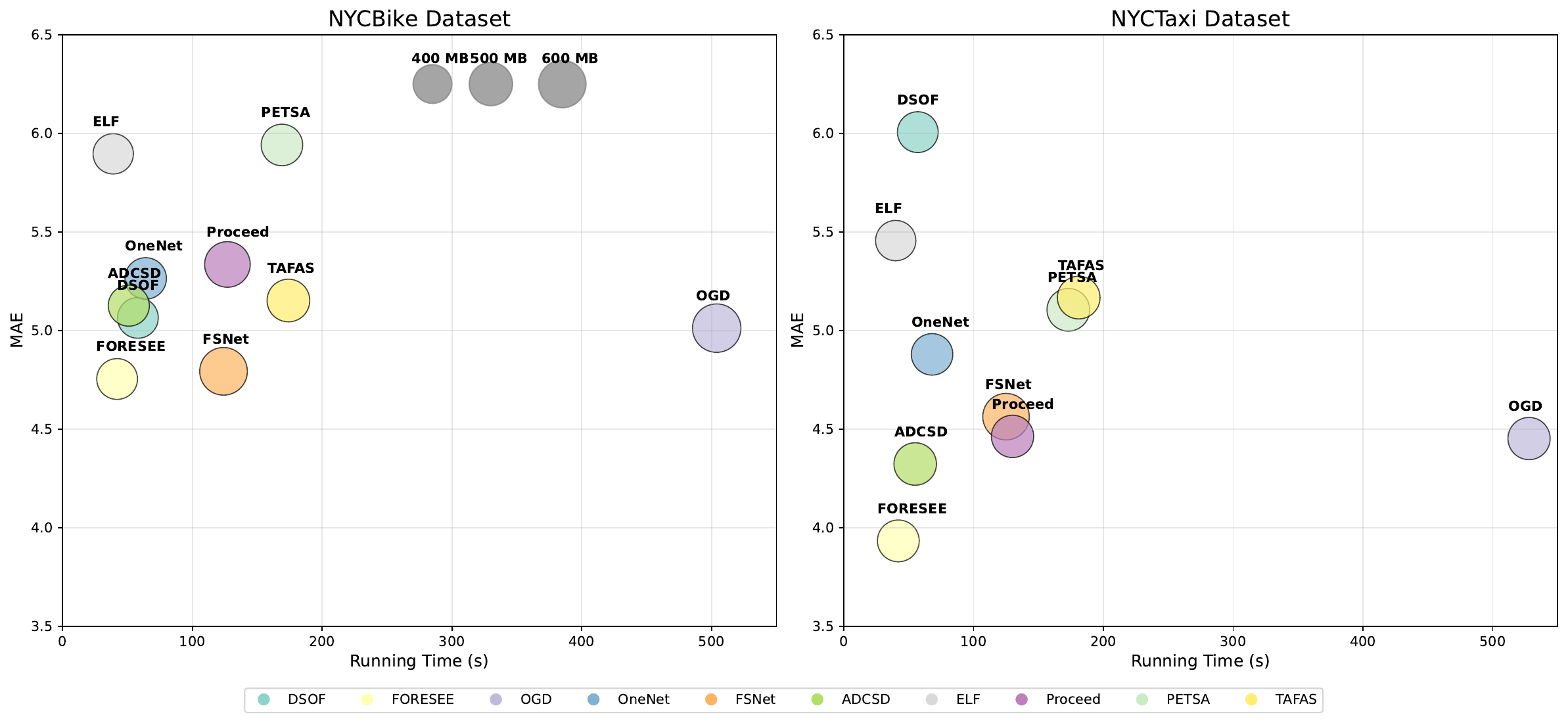}
    \caption{Model efficiency comparison under NYCBIKE and NYCTAXI dataset}
    \label{fig:eff}
\end{figure}

From Figure \ref{fig:eff}, we can see that the running time of our method is shorter than most baseline methods. For example, on the bike dataset, our running time is less than 50 seconds, while some methods require over two minutes. The reason is that our method requires less computation. Some baseline methods need to backpropagate gradients to update the entire model parameters, which demands much more computation. But our method does not require this.

In terms of GPU memory usage, the difference is not very obvious. However, it can still be observed that our method uses slightly less memory than methods like FSNet and OGD. It is worth noting that since this prediction model (GWNET) is a traditional deep learning approach, the model itself does not require much memory—the maximum usage is only around 600MB to 700MB. Therefore, the difference in memory consumption is not significant. However, if a very large model were used, methods that require updating all model parameters or back propagating gradients through deep networks would theoretically need much more GPU memory. In contrast, our method does not require gradient backpropagation at the prediction model level, so it should theoretically use significantly less memory than baseline methods.
\section{Conclusion}
In summary, this paper proposes an accurate, robust, and efficient online deployment method for traffic prediction models. \textbf{Accuracy} is reflected in our method's ability to significantly improve the performance of the original model during online deployment, enabling it to adapt to evolving traffic patterns. Compared to existing online time-series prediction methods, our approach is highly competitive, achieving the best prediction accuracy in most cases. \textbf{Robustness} is demonstrated by the fact that even on datasets with minimal changes in traffic patterns—where some existing online methods fail due to overfitting and perform worse than the original model—our method adaptively adjusts the weights of experts to achieve the accuracy that is at least on par with the original model. \textbf{Efficiency} is achieved as our approach requires no gradient computation or parameter updates to the original model. This results in lower GPU memory usage, making it suitable for use with foundation models. Furthermore, our method requires significantly less computation time than some existing approaches.

In conclusion, this work provides a practical and efficient solution for the real-world application of traffic prediction models. It paves the way for more adaptive and scalable intelligent transportation systems by enabling models to continuously learn and improve online with minimal computational overhead.

\bibliographystyle{elsarticle-harv} 
\bibliography{ref}
\newpage
\appendix
\section{More implemental details}
\subsection{Base prediction model}
For the two traditional models, SDGCN and GWNET, we used the official code from an open-source spatiotemporal forecasting project library \citep{10.1145/3474717.3483923}, along with their default hyperparameters. We trained these models on the training set for 100 epochs using the Adam optimizer with an initial learning rate of 0.005. Training was stopped early if the validation error did not decrease for ten consecutive epochs.

As for the base model OpenCity, we used the official pre-trained weights provided by the authors \footnote{https://github.com/HKUDS/OpenCity/tree/main}. However, when deploying it on each dataset, we found that directly using the pre-trained model yielded relatively poor performance. Therefore, we performed additional fine-tuning on the last layer of the pre-trained model for 10 epochs with a learning rate of 0.001.
\subsection{Baselines}

A few baseline methods require additional explanation for their implementation. Because they were originally designed for time series forecasting with specific model architectures, we needed to adapt them to fit our current problem.

First is the OneNet \citep{wen2023onenet} method. It is an online forecasting approach for time series. Its core idea is to calibrate two models on the training set: the first model predicts based solely on historical data of each dimension, and the second predicts based solely on relationships between different variables. During deployment, it continuously adjusts the weights between these two models based on their actual performance to achieve an adaptive forecasting effect. In our implementation, we similarly calibrated two models: a pure LSTM model (capturing only historical information) and a GNN model (capturing only inter-regional relationships). We then used the adaptive weighting strategy proposed in the original OneNet paper for online prediction. It's important to note that the OneNet method is not model-agnostic; its result is tied to the specific model pair used. Therefore, for each dataset, this method produces only one result. In contrast, other model-agnostic methods yield three results per dataset (one for each of our three base forecasting models: STGCN, GWNET, etc.). For clarity in our results table, we report OneNet's result in the column corresponding to the STGCN model, leaving the entries for the other two base models blank.

Next is the FSNet method. It was originally proposed specifically for models with Temporal Convolutional Networks (TCN). In our implementation, we adapted it to work with STGCN and GWNET. Specifically, FSNet dynamically adjusts the parameters of the convolutional layers within a model and the features extracted by these layers. Accordingly, in our practice, we followed the paper's methodology to dynamically adjust the parameters of the temporal convolutional layers and their corresponding output features in these two spatiotemporal forecasting models.
\section{Proofs}
Proof of Theorem \ref{Tho:ori}
\begin{proof}
\begin{equation}
    \mathbb{E}\sum_{t=1}^{T}\left|\left|\hat{y}_t-y_t\right|\right|_2^2=\mathbb{E}\sum_{t=1}^{T}\left|\left|f(x_t)-f(x_t)+g_t+\epsilon_t\right|\right|_2^2
\end{equation}
    
Because $E\epsilon_t=0$ and independent of $g_t$, Then:
\begin{equation}
    \sum_{t=1}^{T}\left|\left|f(x_t)-f(x_t)+g_t+\epsilon_t\right|\right|_2^2=\sum_{t=1}^{T}\left|\left|g_t\right|\right|_2^2+\sum_{t=1}^{T}{Tr(\Sigma_t)}
\end{equation}

\end{proof}
Proof of Theorem \ref{the:pro}
\begin{proof}
For any $t$:
    \begin{align}
        \mathbb{E}\left|\left|\hat{y}_t-y_t\right|\right|_2^2&=\mathbb{E}\left|\left|f\left(x_t\right)+\left(y_{t-1}-f\left(x_{t-1}\right)\right)-f\left(x_t\right)-g_t-\epsilon_t\right|\right|_2^2 \\&=\mathbb{E}\left|\left|f\left(x_t\right)+\left(f\left(x_{t-1}\right)+g_{t-1}+\epsilon_{t-1}-f\left(x_{t-1}\right)\right)-f\left(x_t\right)-g_t-\epsilon_t\right|\right|_2^2
\\&=\mathbb{E}\left|\left|g_{t-1}-g_t+\epsilon_t+\epsilon_{t-1}\right|\right|_2^2=\mathbb{E}\left|\left|v_t\right|\right|_2^2+2Tr(\Sigma_t)
    \end{align}
    This proof can be completed by conducting summation form $t=1$ to $T-1 $for the above equation.
\end{proof}
Proof of Theorem \ref{theo:ema}
\begin{proof}
    Let us define:
    \begin{equation}
        e_t=\delta_t-g_t
    \end{equation}
    As a result:
    \begin{equation}
        \mathbb{E}||\hat{y}_t-y_t||^2_2=\mathbb{E}e_t^2+Tr\left(\Sigma_t\right)
    \end{equation}
    Then the error in test set is:
    \begin{equation}
        \mathbb{E}\sum_{t=1}^{T}||\hat{y}_t-y_t||^2_2=\mathbb{E}\sum_{t=1}^{T}e_t^2+\sum_{t=1}^{T}{{Tr(\Sigma}_t)}
    \end{equation}
    Because we use EMA to find $\delta_{t+1}$ as:
    \begin{equation}
        \delta_{t+1}=\alpha\delta_t+\left(1-\alpha\right)\left(y_t-f\left(x_t\right)\right)
    \end{equation}
    Then we have:
    \begin{equation}
        e_{t+1}+g_{t+1}=\alpha\left(e_t+g_t\right)+(1-\alpha)(g_t+\epsilon_t)
    \end{equation}
    As a result:
    \begin{equation}
        e_{t+1}=\alpha e_t+\left(g_t-g_{t+1}\right)+\left(1-\alpha\right)\epsilon_t
        \label{eq:et+1}
    \end{equation}
    Then:
    \begin{equation}
        \mathbb{E}e_{t+1}^2=\mathbb{E}\left(\alpha^2e_t^2+\left|\left|g_t-g_{t+1}\right|\right|^2+\left(1-\alpha\right)\epsilon_t^2+2\alpha\left(g_t-g_{t+1}\right)e_t\epsilon_t\right)
    \end{equation}
    Conducting summation form $t=1$ to $T-1$ we have:
    \begin{equation}
        \mathbb{E}\sum_{t=1}^{T-1}e_{t+1}^2=\mathbb{E}\left(\sum_{t=1}^{T-1}{\alpha^2e_t^2}+\sum_{t=1}^{T-1}\left|\left|g_t-g_{t+1}\right|\right|^2+\sum_{t=1}^{T-1}{\left(1-\alpha\right)\epsilon_t^2}+\sum_{t=1}^{T-1}{2\alpha\left(g_t-g_{t+1}\right)e_t\epsilon_t}\right)
    \end{equation}
    Then for large $T$ we omit the first the last term $e_1^2$ and $e_T^2$.
    \begin{equation}
        \mathbb{E}\left(1-\alpha^2\right)\sum_{t=1}^{T}e_t^2=\sum_{t=1}^{T-1}{\left|\left|g_t-g_{t+1}\right|\right|^2+\sum_{t=1}^{T-1}{\left(1-\alpha\right){Tr(\Sigma}_t)}+\mathbb{E}\sum_{t=1}^{T-1}{2\alpha\left(g_t-g_{t+1}\right)e_t\epsilon_t}}
    \end{equation}
    Then we need to tackle the last term $\sum_{t=1}^{T-1}{2\alpha\left(g_t-g_{t+1}\right)e_t\epsilon_t}$. Recalling that $\delta_{t+1}=\alpha\delta_t+\left(1-\alpha\right)\left(y_t-f\left(x_t\right)\right)$ and Equation \ref{eq:et+1}, we can get:
    \begin{equation}
        e_{t+1}=-\sum_{i=1}^{t}\alpha^i\left(g_{t+1-i}-g_{t-i}\right)+(1-\alpha)\sum_{i=1}^{t}{\alpha^i\epsilon_{t-i}}
    \end{equation}
    Then
   \begin{align}
       \mathbb{E}2\alpha\left(g_t-g_{t+1}\right)e_t\epsilon_t&=2\alpha<-\sum_{i=1}^{t-1}\alpha^i\mathbb{E}\left(g_{t+1-i}-g_{t-i}\right),\left(g_t-g_{t+1}\right)>\\&+2\alpha \mathbb{E}<\left(1-\alpha\right)\sum_{i=1}^{t}{\alpha^i\epsilon_{t-i}},\left(g_t-g_{t+1}\right)>
   \end{align}
   Both terms are 0 because $g_{t+1}-g_t=v_t$ and $v_t$ is mean zero and independent. 
   As a result:
   \begin{equation}
       \mathbb{E}e_{t+1}^2=\frac{\sum_{t=1}^{T-1}\left|\left|g_t-g_{t+1}\right|\right|^2}{1-\alpha^2}+\frac{\sum_{t=1}^{T-1}{{Tr(\Sigma}_t)}}{1+\alpha}
   \end{equation}
   and 
   \begin{equation}
       \mathbb{E}\sum_{t=1}^{T}\left(pred_t-y_t\right)^2=\frac{\sum_{t=1}^{T-1}\left|\left|v_{t+1}\right|\right|^2}{1-\alpha^2}+\frac{(2+\alpha)\sum_{t=1}^{T-1}{{Tr(\Sigma}_t)}}{1+\alpha}
   \end{equation}
\end{proof}
Proof of Theorem \ref{theo:emas}
\begin{proof}
If $l(y,\hat{y})$ is defined as $l(y,\hat{y})=||y-\hat{y}||_2^2$ and
    let $u_{i,t}=u_{i,t}e^{-\eta l(y_{i,t},{\hat{y}}_{i,t})}$ and $u_{i,1}=1$.
    We first prove for any t:
    \begin{equation}
        w_{i,t}=\frac{u_{i,t}}{\sum_{j=1}^{k}u_{j,t}}
        \label{eq:w2u}
    \end{equation}
    It is easy to see that this equation is true in $t=1$.
    
Suppose in time $t$ this equation is true, then in $t+1$
\begin{align}
    w_{i,t+1}=&\frac{w_{i,t}\exp(-\eta\left({\hat{y}}_{i,t}-y_t\right)^2)}{\sum_{j}{w_{j,t}\exp(-\eta}\left({\hat{y}}_{j,t}-y_t\right)^2)}=\frac{\frac{u_{i,t}}{\sum_{j=1}^{k}u_{j,t}}\exp(-\eta\left({\hat{y}}_{i,t}-y_t\right)^2)}{\sum_{j}{\frac{u_{j,t}}{\sum_{l=1}^{k}u_{l,t}}\exp(-\eta}\left({\hat{y}}_{j,t}-y_t\right)^2)}\\&=\frac{u_{i,t}\exp(-\eta\left({\hat{y}}_{i,t}-y_t\right)^2)}{\sum_{j}{u_{j,t}\exp(-\eta}\left({\hat{y}}_{j,t}-y_t\right)^2)}=\frac{u_{i,t+1}}{\sum_{j=1}^{k}u_{j,t+1}}
\end{align}
Then we can conclude in any $t$, Equation \ref{eq:w2u} is true.
Then we define:
\begin{equation}
    \Phi_t=\log(\sum_{i=1}^{k}u_{i,t})
\end{equation}
Then:
\begin{equation}
    \Phi_{t+1}-\Phi_t=\log{\left(\frac{\sum_{i=1}^{k}{v_{i,t}e^{-\eta l\left(y_{i,t},{\hat{y}}_{i,t}\right)}}}{\sum_{i=1}^{k}v_{i,t}}\right)}=\log(\sum_{i=1}^{k}{e^{-\eta X}\frac{u_{i,t}}{\sum_{j=1}^{k}u_{j,t}}})=\log(\mathbb{E}_{p_t}e^{-\eta X})
\end{equation}
Where $X=l\left(y_{i,t},{\hat{y}}_{i,t}\right)$. Then $X$ is in $[0,B]$ and $p_t$ is the distribution of experts with $p_t\left(i\right)=\frac{u_{i,t}}{\sum_{j=1}^{k}u_{j,t}}$. As a result:
\begin{equation}
    \mathbb{E}_{p_t}e^{-\eta X}=\mathbb{E}_{p_t}e^{-\eta(X-\mathbb{E}_{p_t}X)}e^{\eta\mathbb{E}_{p_t}X}
\end{equation}
Then
\begin{equation}
    \Phi_{t+1}-\Phi_t=log\mathbb{E}_{p_t}e^{-\eta\left(X-\mathbb{E}_{p_t}X\right)}+\eta\mathbb{E}_{p_t}X
\end{equation}
By Hoeffding’s lemma (lemma 2.2 in \citep{Concentration2013}): for any random variable $X$ bounded by $[0,B]$
\begin{equation}
    \mathbb{E}_{p_t}\left[e^{\eta(X-\mathbb{E}_{p_t}X)}\right]\le e^\frac{B^2\eta^2}{8}
\end{equation}
Then:
\begin{equation}
    \Phi_{t+1}-\Phi_t\le-\eta\mathbb{E}_{p_t}X+\frac{B^2\eta^2}{8}
\end{equation}
Because $l(y_{i,t},{\hat{y}}_{i,t})$ is convex of  ${\hat{y}}_{i,t}$ and we use Jeson’s inequality:
\begin{equation}
    \eta\mathbb{E}_{p_t}X=\eta\mathbb{E}_{p_t}l\left(y_{i,t},{\hat{y}}_{i,t}\right)\geq\eta l\left(\mathbb{E}_{p_t}\hat{y}_{i,t},y_{i,t}\right)=\eta l\left({\hat{y}}_t,y_{i,t}\ \right)
\end{equation}
Then:
\begin{equation}
    \Phi_{t+1}-\Phi_t\le-\eta l\left({\hat{y}}_t,y_{i,t}\ \right)+\frac{B^2\eta^2}{8}
\end{equation}
Take summation of the above inequality over 1 to $T$:
\begin{equation}
    \sum_{t=1}^{T}{\Phi_{t+1}-\Phi_t}\le-\eta\sum_{t=1}^{T}l\left({\hat{y}}_t,y_{i,t}\ \right)+\frac{{TB}^2\eta^2}{8}
\end{equation}
Then
\begin{equation}
    \Phi_{T+1}-\Phi_1\le-\eta\sum_{t=1}^{T}l\left({\hat{y}}_t,y_{i,t}\ \right)+\frac{{TB}^2\eta^2}{8}
\end{equation}
Because $\Phi_1=logk$, as a result,$ \Phi_{T+1}\le-\eta\sum_{t=1}^{T}l\left({\hat{y}}_t,y_{i,t}\ \right)+\frac{{TB}^2\eta^2}{8}+logk$. Besides, $\Phi_{T+1}=\log{\left(\sum_{i=1}^{k}v_{i,T+1}\right)}$ and $u_{i,T+1}=e^{-\eta\sum_{t=1}^{T}l\left({\hat{y}}_{i,t},y_{i,t}\right)}$. Let we define $l_{i,T}=\sum_{t=1}^{T}l\left({\hat{y}}_{i,t},y_{i,t}\right)$ be the cumulative loss of expert $i$ up to time $t$, Then
\begin{equation}
    \Phi_{T+1}=\log{\left(\sum_{i=1}^{k}e^{-\eta l_{i,T}}\right)}\geq\log(e^{-\eta\min_i{l_{i,T}}})
\end{equation}
Then we combine the above inequality with the upper bound of $\Phi_{T+1}$
\begin{equation}
    -\eta\min_i{l_{i,T}}\le-\eta\sum_{t=1}^{T}l\left({\hat{y}}_t,y_{i,t}\ \right)+\frac{{TB}^2\eta^2}{8}+logk
\end{equation}
As a result:
\begin{equation}
    \sum_{t=1}^{T}l\left({\hat{y}}_t,y_{i,t}\right)-\min_i{l_{i,T}}\le\frac{{TB}^2\eta}{8}+\frac{logk}{\eta}\le\sqrt{\frac{1}{2}TB^2logk}
\end{equation}
Then this theorem can be proved.
\end{proof}
\end{document}